\documentclass{article}

\usepackage[utf8]{inputenc} % allow utf-8 input
\usepackage[T1]{fontenc}    % use 8-bit T1 fonts
\usepackage{hyperref}       % hyperlinks
\usepackage{url}            % simple URL typesetting
\usepackage{booktabs}       % professional-quality tables
\usepackage{amsfonts}       % blackboard math symbols
\usepackage{nicefrac}       % compact symbols for 1/2, etc.
\usepackage{microtype}      % microtypography

\usepackage{multirow}
\usepackage{graphicx}
\usepackage{float}

\usepackage{algorithm,algpseudocode}
\usepackage{cancel}
\usepackage{goku}
\usepackage{caption}
\usepackage{hyperref}

\usepackage{float}

\usepackage[accepted]{icml2020}

\def\hA{\widehat{A}}
\def\hSigma{\widehat{\Sigma}}

\renewcommand{\P}{\mathbb{P}}
\def\reals{\mathbf{R}}
\def\symm{\mathbf{S}}
\newcommand{\stdev}{\textrm{SD}}

\newcommand{\E}{\mathbb{E}}

\icmltitlerunning{Two Simple Ways to Learn Individual Fairness Metrics from Data}

\begin{document}

\twocolumn[
\icmltitle{Two Simple Ways to Learn Individual Fairness Metrics from Data}
% Learning fairness metrics from human supervision

% It is OKAY to include author information, even for blind
% submissions: the style file will automatically remove it for you
% unless you've provided the [accepted] option to the icml2020
% package.

% List of affiliations: The first argument should be a (short)
% identifier you will use later to specify author affiliations
% Academic affiliations should list Department, University, City, Region, Country
% Industry affiliations should list Company, City, Region, Country

% You can specify symbols, otherwise they are numbered in order.
% Ideally, you should not use this facility. Affiliations will be numbered
% in order of appearance and this is the preferred way.
\icmlsetsymbol{equal}{*}

\begin{icmlauthorlist}
\icmlauthor{Debarghya Mukherjee}{um,equal}
\icmlauthor{Mikhail Yurochkin}{ibm,equal}
\icmlauthor{Moulinath Banerjee}{um}
\icmlauthor{Yuekai Sun}{um}
% \icmlauthor{Fiuea Rrrr}{to}
% \icmlauthor{Tateu H.~Yasehe}{ed,to,goo}
% \icmlauthor{Aaoeu Iasoh}{goo}
% \icmlauthor{Buiui Eueu}{ed}
% \icmlauthor{Aeuia Zzzz}{ed}
% \icmlauthor{Bieea C.~Yyyy}{to,goo}
% \icmlauthor{Teoau Xxxx}{ed}
% \icmlauthor{Eee Pppp}{ed}
\end{icmlauthorlist}

\icmlaffiliation{um}{Department of Statistics, University of Michigan}
\icmlaffiliation{ibm}{IBM Research, MIT-IBM Watson AI Lab}

% \icmlcorrespondingauthor{Cieua Vvvvv}{c.vvvvv@googol.com}
% \icmlcorrespondingauthor{Eee Pppp}{ep@eden.co.uk}
\icmlcorrespondingauthor{Debarghya Mukherjee}{mdeb@umich.edu}

% You may provide any keywords that you
% find helpful for describing your paper; these are used to populate
% the "keywords" metadata in the PDF but will not be shown in the document
\icmlkeywords{Machine Learning, ICML}

\vskip 0.3in
]

%his must go after the closing bracket ] following \twocolumn[ ...

% This command actually creates the footnote in the first column
% listing the affiliations and the copyright notice.
% The command takes one argument, which is text to display at the start of the footnote.
% The \icmlEqualContribution command is standard text for equal contribution.
% Remove it (just {}) if you do not need this facility.

%\printAffiliationsAndNotice{}  % leave blank if no need to mention equal contribution
\printAffiliationsAndNotice{\icmlEqualContribution} % otherwise use the standard text.

\newcommand{\redcmt}[1]{{\red{\bf DM:} #1}}
\begin{abstract}
Individual fairness is an intuitive definition of algorithmic fairness that addresses some of the drawbacks of group fairness. Despite its benefits, it depends on a task specific fair metric that encodes our intuition of what is fair and unfair for the ML task at hand, and the lack of a widely accepted fair metric for many ML tasks is the main barrier to broader adoption of individual fairness. In this paper, we present two simple ways to learn fair metrics from a variety of data types. We show empirically that fair training with the learned metrics leads to improved fairness on three machine learning tasks susceptible to gender and racial biases. We also provide theoretical guarantees on the statistical performance of both approaches. 
\end{abstract}

\section{Introduction}

Machine learning (ML) models are an integral part of modern decision-making pipelines. They are even part of some high-stakes decision support systems in criminal justice, lending, medicine \etc{}. Although replacing humans with ML models  in the decision-making process appear to eliminate human biases, there is growing concern about ML models reproducing historical biases against certain historically disadvantaged groups. This concern is not unfounded. For example, \citet{dastin2018Amazon} reports gender-bias in Amazon's resume screening tool, \citet{angwin2016Machine} mentions racial bias in recidivism prediction instruments, \citet{vigdor2019Apple} reports gender bias in the credit limits of Apple Card.

As a first step towards mitigating algorithmic bias in ML models, researchers proposed a myriad of formal definitions of algorithmic fairness. At a high-level, there are two groups of mathematical definitions of algorithmic fairness: group fairness and individual fairness. Group fairness divides the feature space into (non-overlapping) protected subsets and imposes invariance of the ML model on the subsets. Most prior work focuses on group fairness because it is amenable to statistical analysis. Despite its prevalence, group fairness suffers from two critical issues. First, it is possible for an ML model that satisfies group fairness to be blatantly unfair with respect to subgroups of the protected groups and individuals \cite{dwork2011Fairness}. Second, there are fundamental incompatibilities between seemingly intuitive notions of group fairness \cite{kleinberg2016Inherent,chouldechova2017Fair}. 

In light of the issues with group fairness, we consider individual fairness in our work. Intuitively, individually fair ML models should treat similar users similarly.  \citet{dwork2011Fairness} formalize this intuition by viewing ML models as maps between input and output metric spaces and defining individual fairness as Lipschitz continuity of ML models. The metric on the input space is the crux of the definition because it encodes our intuition of which users are similar. Unfortunately, individual fairness was dismissed as impractical because there is no widely accepted similarity metric for most ML tasks. In this paper, we take a step towards operationalizing individual fairness by showing it is possible to learn good similarity metrics from data. 

The rest of the paper is organized as follows. In Section \ref{sec:dataDrivenMetric},  we describe two different ways to learn data-driven fair metric: one from knowledge of groups of similar inputs and another from knowledge of similar and dissimilar pairs of inputs. In Section \ref{sec:theory}, we show that (i) the methods are robust to noise in the data, and (ii) the methods leads to individually fair ML models. Finally, in Section \ref{sec:computationalResults}, we demonstrate the effectiveness of the methods in mitigating bias on two ML tasks susceptible to gender and racial biases.

\section{Learning fair metrics from data}
\label{sec:dataDrivenMetric}

The intuition underlying individual fairness is fair ML models should treat \emph{comparable} users similarly. We write \emph{comparable} instead of \emph{similar} in the rest of this paper to emphasize that comparable samples may differ in ways that are irrelevant to the task at hand. Formally, we consider an ML model as a map $h:\cX\to\cY$, where $(\cX,d_x)$ and $(\cY,d_y)$ are the input and output metric spaces respectively. Individual fairness \cite{dwork2011Fairness,friedler2016im} is $L$-Lipschitz continuity of $h$:
\begin{equation}
d_y(h(x_1),h(x_2)) \le Ld_x(x_1,x_2)\text{ for all }x_1,x_2\in\cX.
\label{eq:individualFairness}
\end{equation}
The choice of $d_{\cY}$ depends on the form of the output. For example, if the ML model outputs a vector of the logits, then we may pick the Euclidean norm as $d_{\cY}$ \cite{kannan2018Adversarial,garg2018Counterfactual}. The fair metric $d_x$ is the crux of the definition. It encodes our intuition of which samples are comparable; \ie\ which samples only differ in ways that are irrelevant to the task at hand. Originally, \citet{dwork2011Fairness} deferred the choice of $d_x$ to regulatory bodies or civil rights organizations, but we are unaware of widely accepted fair metrics for most ML tasks. This lack of widely accepted fair metrics has led practitioners to dismiss individual fairness as impractical. Our goal here is to address this issue by describing two ways to  learn fair metrics from data.

We start from the premise there is generally more agreement than disagreement about what is fair in many application areas. For example, in natural language processing, there are ways of identifying groups of training examples that should be treated similarly \cite{bolukbasi2016Man,madaan2018Analyze} or augmenting the training set with hand-crafted examples that should be treated similarly as observed training examples \cite{garg2019counterfactual}. Even in areas where humans disagree, there are attempts to summarize the cases on which humans agree in metrics by fitting metrics to human feedback \cite{wang2019Empirical}. Our goal is similar: encode what we agree on in a metric, so that we can at least mitigate the biases that we agree on with methods for enforcing individual fairness  \cite{kim2018Fairness,rothblum2018Probably,yurochkin2020Training}. 

To keep things simple, we focus on fitting metrics of the form
\begin{equation}
d_x(x_1,x_2) \triangleq \langle\varphi(x_1) - \varphi(x_2),\Sigma(\varphi(x_1) - \varphi(x_2))\rangle,
\label{eq:generalizedMahalanobisDistance}
\end{equation}
where $\varphi(x):\cX\to\reals^d$ is an embedding map and $\Sigma\in\symm_+^d$. The reason behind choosing Mahalanobis distance is that the learned feature maps (\eg\ the activations of the penultimate layer of a deep neural network) typically map non-linear structures in the raw feature space to linear structures in the learned feature space \cite{mikolov2013Distributed,radford2015Unsupervised,brock2018Large}. To keep things simple, we assume $\varphi$ is known and learn the matrix $\Sigma$ from the embedded observations $\varphi's$. The data may consist of human feedback, hand-picked groups of similar training examples, hand-crafted examples that should be treated similarly as observed training examples, or a combination of the above. In this section, we describe two simple methods for learning fair metrics from diverse data types.

% Since we are obtaining feedback from humans, it is imperative that we allow for errors and disagreements among the humans. To address this issue, we adopt a \textit{probabilistic} approach to the task that avoids overfitting the feedback provided by any one (or small group of) human(s).

% We recognize it is hard for humans to quantify their intuition precisely. For example, it is unrealistic to expect humans to provide an exact numerical value for the comparability of two samples. It is equally unrealistic to expect humans to compare large numbers of samples at once (\eg\ find the closest sample to a given sample). To keep the supervision accurate and consistent, we consider two forms of human supervision:
% \begin{enumerate}
% \item \textit{comparable groups:} groups of comparable samples. For example, to remove gender bias in natural language processing (NLP) systems, the humans may provide a set of common male and female names as a group of comparable samples. 
% \item \textit{pair-wise comparisons:} pairs of comparable and incomparable samples. This is a more nuanced form of supervision because the humans also provide negative feedback. 
% \end{enumerate}

\subsection{FACE: Factor Analysis of Comparable Embeddings}
\label{sec:factorAnalysis}

In this section, we consider learning $\Sigma$ from groups of comparable samples. The groups may consist of hand-picked training examples \cite{bolukbasi2016Man,madaan2018Analyze} or hand-crafted examples that differ in certain ``sensitive'' ways from observed training examples \cite{garg2019counterfactual}. 
% For example, \citeauthor{bolukbasi2016Man}'s method for mitigating gender bias in word embeddings relies on hand-picked sets of words whose embeddings mainly vary in a gender direction. 

To motivate the approach, we posit the embedded features satisfy a factor model:
\begin{equation}
\varphi_i = A_*u_i + B_*v_i + \eps_i
\label{eq:FACE-model}
\end{equation}
where $\varphi_i\in\reals^d$ is the learned representation of $x_i$, $u_i\in\reals^K$ (resp.\ $v_i\in\reals^L$) is the protected/sensitive (resp.\ discriminative/relevant) attributes of $x_i$ for the ML task at hand, and $\eps_i$ is an error term. A pair of samples are comparable if their (unobserved) relevant attributes are similar. 
% In practice, the learned representations are observed, the sensitive attribute may be observed, and the relevant attributes are not observed. 
For example, \citeauthor{bolukbasi2016Man}'s method for mitigating gender bias in word embeddings relies on word pairs that only differ in their gender associations (\eg\ (he, she), (man, woman), (king, queen) \etc). 

The factor model \eqref{eq:FACE-model} decomposes the variance of the learned representations into variance due to the sensitive attributes and variance due to the relevant attributes. We wish to learn a metric that ignores the variance attributed to the sensitive attributes but remains sensitive to the variance attributed to the relevant attributes. This way, the metric declares any pair of samples that differ mainly in their sensitive attributes as comparable. One possible choice of $\Sigma$ is the projection matrix onto the orthogonal complement of $\ran(A_*)$, where $\ran(A_*)$ is the column space of $A_*$. Indeed,
\[
\begin{aligned}
d_x(x_1,x_2) &= \langle\varphi_1 - \varphi_2,(I - P_{\ran(A_*)})(\varphi_1 - \varphi_2)\rangle \\
&\approx \langle B_*(v_1 - v_2),(I - P_{\ran(A_*)})B_*(v_1 - v_2)\rangle,
\end{aligned}
\]
which ignores differences between $\varphi_1$ and $\varphi_2$ due to differences in the sensitive attributes.  Although $\ran(A_*)$ is unknown, it is possible to estimate it from the learned representations and groups of comparable samples by factor analysis (see Algorithm \ref{alg:face}). We remark that our target is $\ran(A_*)$, not $A_*$ itself. This frees us from cumbersome identification restrictions common in the factor analysis literature. 

\begin{algorithm}[H]
  \caption{estimating $\ran(A_*)$ by factor analysis}
  \label{alg:face}
  \begin{algorithmic}[1]
    \State {\bf Input:} $\{\varphi_i\}_{i=1}^n$, comparable groups $\cI_1,\dots,\cI_G$
    \State $\hA^T \in \argmin_{W_g,A}\{{\textstyle\frac12\sum_{g=1}^G\|H_g\Phi_{\cI_g} - W_gA^T\|_F^2}\}$,
    where $H_g \triangleq I_{\abs{\cI_g}} - \frac{1}{\abs{\cI_g}}1_{\abs{\cI_g}}1_{\abs{\cI_g}}^T$ is the centering matrix
    \State $Q \gets \texttt{qr}(\hA)$ 
    \Comment{get orthonormal basis of $\ran(\hA)$}
    \State $\hSigma \gets I_d - QQ^T$
  \end{algorithmic}
\end{algorithm}

Algorithm \ref{alg:face} is based on the observation that groups of comparable samples have similar relevant attributes; \ie\ 
\begin{equation}
\begin{aligned}
H\Phi_{\cI} &= HU_{\cI}A_*^T + \cancelto{\approx 0}{HV_{\cI}B_*^T} + HE_{\cI} \\
&\approx HU_{\cI}A_*^T + HE_{\cI},
\end{aligned}
\label{eq:centersubGroup}
\end{equation}
where $H \triangleq I_{\abs{\cI}} - \frac{1}{\abs{\cI}}1_{\abs{\cI}}1_{\abs{\cI}}^T$ is the centering matrix and $\Phi_{\cI}$ (resp.\ $U_{\cI}$, $V_{\cI}$) is the matrix whose rows are the  $\varphi_i$'s (resp.\ $u_i$'s, $v_i$'s). 
% Ideally, the relevant attributes of all samples in a group are identical, \ie\ $V_{\cI} = 1_{\abs{\cI}}v^T$ for some $v$, and $HV_{\cI}$ vanishes exactly. In this case, \eqref{eq:centersubGroup} implies 
% \[
% \Ex\big[\Phi_{\cI}^TH\Phi_{\cI}\big] \approx A\Ex\big[U_{\cI}^THU_{\cI}\big]A^T + \Ex\big[E_{\cI}^THE_{\cI}\big]
% \] 
This is the factor model that Algorithm \ref{alg:face} fits in Step 2 to obtain $\hA$ whose range is close to that of $\hA$. In Steps 3 and 4, the algorithm forms the projector onto the orthogonal complement of $\ran(\hA)$.

\subsection{EXPLORE: Embedded Xenial Pairs Logistic Regression}
\label{matrix-logistic}
EXPLORE learns a fair metric from pair-wise comparisons. More concretely, the data comes from human feedback in the form of triplets $\{(x_{i_1},x_{i_2},y_i)\}_{i=1}^n$, where $y_i\in\{0,1\}$ indicates whether the human considers $x_{i_1}$ and $x_{i_2}$ comparable ($y_i = 1$ indicates comparable). We posit $(x_{i_1},x_{i_2},y_i)$ satisfies a binary response model.
\begin{equation}
\begin{aligned}
y_i\mid x_{i_1},x_{i_2} &\sim \Ber(2\sigma(-d_i)),\\
d_i &\triangleq \|\varphi_{i_1} - \varphi_{i_2}\|_{\Sigma_0}^2 \\
&= (\varphi_{i_1} - \varphi_{i_2})^T\Sigma_0(\varphi_{i_1} - \varphi_{i_2}) \\
&= \langle(\underbrace{\varphi_{i_1} - \varphi_{i_2})(\varphi_{i_1} - \varphi_{i_2})^T}_{D_i},\Sigma_0\rangle
\label{eq:binaryResponseModel}
\end{aligned}
\end{equation}
where $\sigma(z) \triangleq \frac{1}{1 + e^{-z}}$ is the logistic function, $\varphi_{i_1}$ (resp.\ $\varphi_{i_2}$) is the learned representations of $x_{i_1}$ (resp.\ $x_{i_2}$), and $\Sigma_0\in\symm_+^d$. The reason for multiplying by 2 is to make $P(y_i =1|x_{i_1}, x_{i_2})$ close to 1 when $\varphi_{i_1}$ is close to $\varphi_{i_2}$ with respect to this scaled distance. This ensures that if we have two comparable samples, then the corresponding $y_i = 1$ with high probability. To estimate $\Sigma_0$ in EXPLORE from the humans' feedback, we seek the maximum of the log-likelihood
\begin{equation}
\begin{aligned}
\ell_n(\Sigma) &= \frac1n\sum_{i=1}^ny_i\log\frac{2\sigma(-\langle D_i,\Sigma\rangle)}{1-2\sigma(-\langle D_i,\Sigma\rangle)} \\
&\quad+ \log(1-2\sigma(-\langle D_i,\Sigma\rangle)).
\end{aligned}
\label{eq:log-likelihood}
\end{equation}
on $\symm_+^d$. As $\ell_n$ is concave (in $\Sigma$), we appeal to a stochastic gradient descent (SGD) algorithm to maximize $\ell_n$. 
% As the total amount of data is too large to process efficiently, we resort to minibatches, which we chose at the very begining of the algorithm. 
The update rule is
% At step $t+1$, set the step-length $\alpha_{(t)}$ and update $\Sigma$ as: 
$$\Sigma_{t+1} = \texttt{ProjPSD}(\Sigma_t + \eta_t\partial \tilde{\ell}_n(\Sigma_t)),$$
%\label{alg:1.2}
where $\tilde{\ell}_n$ is the likelihood of the $t$-th minibatch, $\eta_t > $ is a step size parameter, and \texttt{ProjPSD} is the projection onto the PSD cone. 

%   \begin{algorithm}[H]
%   \caption{Estimating $\Sigma$ by sgd}
%   \label{alg:1.1}
%   \begin{algorithmic}[1]
%     \State {\bf Input:} $\{y_i, \varphi_{i_1} - \varphi_{i_2}\}_{i=1}^n$, batchsize = B, maximum iteration allowed = M.
%     \State Draw a sample of size B without replacement from the input samples. \\ 
%     \State Start with some randomly selected $\Sigma_{(0)} \in \mathbf{S}^d_{++}$. \\
%     \State{\bf Update:} At step $t+1$, set the step-length $\alpha_{(t)} = 1/(1+ \lfloor t /100\rfloor )$ and $$\Sigma_{(t+1)} = \textbf{ProjPSD}(\tilde{\ell_n}(\Sigma_{(t)}) + \alpha_{(t)} \nabla \tilde{\ell_n}(\Sigma_{(t)}))$$ with $\tilde{\ell}_n$ is the likelihood evaluated on the minibatch.
%     \State {\bf Output:} Run Step 4 $M$ times and produce output $\Sigma_{B}$.
%   \end{algorithmic}
% \end{algorithm}
% For SGD to converge, the step sizes must satisfy $\sum_{t}\alpha(t) = \infty$ and $\sum_t \alpha^2(t) < \infty$ \cite{nemirovski2009Robust}. We present some theoretical properties of both the methods in the following section. 

\subsection{FACE vs EXPLORE}
\label{sec:comparison_conjecture}

At first blush, the choice of which approach to use seems clear from the data. If the data consists of groups of comparable samples, then the factor analysis approach is appropriate. On the other hand, if data consists of pair-wise comparisons, then the logistic-regression approach is more appropriate. However, the type of data is usually part of the design, so the question is best rephrased as which type of data should the learner solicit. As we shall see, if the data is accurate and consistent, then FACE usually leads to good results. However, if the data is noisy, then EXPLORE is more robust.

The core issue here is a bias variance trade-off. Data in the form of a large group of comparable samples is more informative than pair-wise comparisons. As FACE is capable of fully utilizing this form of supervision, it leads to estimates with smaller variance. However, FACE is also more sensitive to heterogeneity within the groups of comparable samples as FACE is fully unbiased if all the variation in the group can be attributed to the sensitive attribute. If some of the variation is due to the discriminative attributes, then FACE leads to biased estimates.
On the other hand, EXPLORE imposes no conditions on the homogeneity of the comparable and incomparable pairs in the training data. While EXPLORE cannot fully utilize comparable groups of size larger than two, it is also more robust to heterogeneity in the pairs of samples in the training data. 

In the end, the key factor is whether it is possible for humans to provide homogeneous groups of comparable samples. In some applications, there are homogeneous groups of comparable samples. For example, in natural language processing, names are a group of words that ought to be treated similar in many ML tasks. For such applications, the factor analysis approach usually leads to better results. In other applications where there is less consensus on whether samples are comparable, the logistic regression approach usually leads to better results. As we shall see, our computational results validate our recommendations here.
% \DM{Discussion on whether we can replace logistic by more agnostic model.}

\subsection{Related work} 

\paragraph{Metric learning} The literature on learning the fair metric is scarce. The most relevant paper is \cite{ilvento2019Metric}, which considers learning the fairness metric from consistent humans. On the other hand, there is a voluminous literature on metric learning in other applications \cite{bellet2013Survey,kulis2013Metric,suarez2018Tutorial,moutafis2017Overview}, including a variety of methods for metric learning from human feedback \cite{frome2007Learning,jamieson2011Lowdimensional,tamuz2011Adaptively,maaten2012Stochastic,wilber2014CostEffective,zou2015Crowdsourcing,jain2016Finite}. The approach described in subsection \ref{sec:factorAnalysis} was inspired by \cite{bolukbasi2016Man,bower2018Debiasing}.

\paragraph{Learning individually fair representations} There is a complementary strand of work on enforcing individual fairness by first learning a fair representation and then training an ML model on top of the fair representation \cite{zemel2013Learning,bower2018Debiasing,madras2018Learning,lahoti2019iFair}. Although it works well on some ML tasks, these methods lack theoretical guarantees that they train individually fair ML models.

\paragraph{Enforcing individual fairness} We envison FACE and EXPLORE as the first stage in a pipeline for training individually fair ML models. The metrics from FACE and EXPLORE may be used in conjunction with methods that enforce individual fairness \cite{kim2018Fairness,rothblum2018Probably,yurochkin2020Training}. There are other methods that enforce individual fairness without access to a metric \cite{Gillen2018Online,jung2019Eliciting}. These methods depend on an oracle that detects violations of individual fairness, and can be viewed as combinations of a metric learning method and a method for enforcing individual fairness with a metric.

\section{Theoretical properties of FACE}
\label{sec:theory}

In this section, we investigate the theoretical properties of FACE. We defer proofs and theoretical properties of EXPLORE to the Appendix.
\subsection{Learning from pairwise comparison}
In this subsection, we establish theory of FACE when we learn the fair metric from comparable pairs. Given a pair $(\varphi_{i,1}, \varphi_{i,2})$ (the embedded version of $(x_{i,1}, x_{i,2})$), define for notational simplicity $z_i = \varphi_{i_1} - \varphi_{i_2}$. Here, we only consider those $z_i$'s which come from a comparable pair, i.e., with corresponding $y_i = 1$. Under our assumption of factor model (see equation \eqref{eq:FACE-model}) we have:
\begin{align}
  z_i & = \varphi_{i_1}-\varphi_{i_2} \notag\\
  & = A_*(u_{i_1}-u_{i_2}) + B_*(v_{i_1}-v_{i_2}) + (\epsilon_{i_1} - \epsilon_{i_2}) \notag\\
    \label{model_equation}   & = A_* \mu_i + B_* \nu_i + w_i
\end{align}
Here we assume that the sensitive attributes have more than one dimension which corresponds to the setting of \emph{intersectional fairness} (\eg\ we wish to mitigate gender and racial bias). We also assume $\mu_i$'s and $\nu_i$'s are isotropic, variance of $w_i$ is $\sigma^2I_d$ and $\mu_i, \nu_i, w_i$ are all independent of each other. The scalings of $\mu_i$ and $\nu_i$ are taken care of by the matrices $A_*$ and $B_*$ respectively. Let $\Sigma_Z$ be covariance matrix of $z_i$'s. From model equation \ref{model_equation} and aforementioned assumptions: 
\begin{equation}
    \label{factor_variance}\Sigma_Z = A_*A_*^T + B_*B_*^T + \sigma^2 I_d
\end{equation}
We assume that we know the dimension of the sensitive direction beforehand which is denoted by $k$ here. As $\phi_{i_1}$ is comparable to $\phi_{i_2}$, we expect that variability along the protected attribute is dominant. Mathematically speaking, we assume $\lambda_{min}(A_* A_*^T) > \|B_*B_*^T + \sigma^2 I_d\|_{op}$. Here the fair metric we try to learn is: 
$$d_x(x_1, x_2) = \langle (\varphi_1 - \varphi_2) , \Sigma_0 (\varphi_1 - \varphi_2) \rangle$$
where $\Sigma_0 = \left(I - P_{\ran(A_*)}\right)$.
To estimate (and hence eliminate) the effect of the protected attribute, we compute the SVD of the sample covariance matrix $S_n = \frac1n \sum_{i=1}^n z_iz_i^{\top}$ of the $z_i$'s and project out the eigen-space corresponding to the top $k$ eigenvectors, denoted by $\hat U$. Our estimated distance metric will be:  
$$\hat d_x(x_1, x_2) = \langle (\varphi_1 - \varphi_2) , \hSigma (\varphi_1 - \varphi_2) \rangle \,,$$
where $\hat \Sigma = \left(I - \hat U \hat U^{\top}\right)$. 
The following theorem quantifies the statistical error of the estimator: 
\begin{theorem}
\label{thm:svd_theorem_new}
Suppose $z_i$'s are centered sub-gaussian random vectors, i.e. $\|z_i\|_{\psi_2} < \infty$ where $\psi_2$ is the Orlicz-2 norm. Then we have with probability at-least $1-2e^{-ct^2}$: 
\allowdisplaybreaks
\begin{equation}
\begin{aligned}
    &\textstyle \|\hat \Sigma - \Sigma_0\|_{op} \le b + \frac{\delta \vee \delta^2}{\tilde{\gamma} - (\delta \vee \delta^2)} 
\end{aligned}
\label{eq:svd_theorem_new}
\end{equation}
for all $t < (\sqrt{n}\tilde{\gamma} - C\sqrt{d}) \wedge (\sqrt{n\tilde{\gamma}} - C\sqrt{d})$, where: 
\begin{enumerate}
% \item  $\tilde{A}_*$ is the collection of the eigenvectors of $A_*$.  
    \item $b = \left(\frac{\lambda_{min}(A_*A_*^T)}{\|B_*B_*^T + \sigma^2 I_d\|_{op}}-1\right)^{-1}$
 
    \item $\delta = \frac{C\sqrt{d} + t}{\sqrt{n}}$. 
    \item $\tilde{\gamma} = \lambda_{min}(A_*A_*^T) - \|B_* B_*^T\|_{op}$.
\end{enumerate}{}
The constants $C,c$ depend only on $\|x_i\|_{\psi_2}$, the Orlicz-2 norm of the $x_i$'s.
\end{theorem}

The error bound on the right side of \eqref{eq:svd_theorem_new} consists of two terms. The first term $b$ is the approximation error/bias in the estimate of the sensitive subspace due to heterogeneity in the similar pairs. Inspecting the form of $b$ reveals that the bias depends on the relative sizes of the variation in the sensitive subspace and that in the relevant subspace: the larger the variation in the sensitive subspace relative to that in the relevant subspace, the smaller the bias. In the ideal scenario where there is no variation in the relevant subspace, Theorem \ref{thm:svd_theorem_new} implies our estimator converges to the sensitive subspace. The second term is the estimation error, which vanishes at the usual $\frac{1}{\sqrt{n}}$-rate. In light of our assumptions on the sub-Gaussianity of the $z_i$'s, this rate is unsurprising. 

\subsection{Learning from group-wise comparisons}
In this subsection, we consider the complementary setting in which we have a single group of $n$ comparable samples. We posit a factor model for the features:
\begin{equation}
\label{eq:FACE_eq_1}
    \varphi_i = m + A_*\mu_i + B_*\nu_i + \epsilon_i \ \ \ i = 1, 2, \dots, n,
\end{equation}
where $m\in\reals^d$ is a mean term that represents the common effect of the relevant attributes in this group of comparable samples, $A_*\mu_i$ represents the variation in the features due to the sensitive attributes, and $B_*\nu_i$ represents any residual variation due to the relevant attributes (\eg\ the relevant attributes are similar but not exactly identical). As before, we assume $\mu_i,\nu_i$'s are isotropic, $\Var(\epsilon_i) = \sigma^2I_d$ and the scale factors of $\mu_i$'s and $\nu_i$'s are taken care of by the matrices $A_*$ and $B_*$ respectively due to identifiability concerns. In other words, the magnitudes of $B_*\nu_i$'s are uniformly small. As the residual variation among the samples in this group due to the relevant factors are small, we assume that $B_*$ is small compared to $A_*$, which can be quantified as before by assuming $\lambda_{\min}(A_*A_*^{\top}) > \|B_*B_*^{\top} + \sigma^2I\|$. Hence to remove the effect of protected attributes, we estimate the column space of $A_*$ from the sample and then project it out. From the above assumptions we can write the (centered) dispersion matrix of $\varphi$ as: 
$$\Sigma_{\phi} = A_*A_*^{\top} + B_*B_*^{\top} + \sigma^2I ,.$$
Note that the structure of $\Sigma_z$ in the previous sub-section is same as $\Sigma_{\varphi}$ as $z$ is merely difference of two $\varphi$'s. As before we assume we know dimension of the protected attributes which is denoted by $k$. Denote (with slight abuse of notation) by $\hat U$, the top $k$ eigenvalues of $S_n = \frac1n \sum_{i=1}^n \varphi_i\varphi_i^{\top}$. Our final estimate of $\Sigma_0$ is $\hSigma = \left(I - \hat U \hat U^{\top}\right)$ and the corresponding estimated fair metric becomes: 
$$d_x(x_1, x_2) = \langle (\varphi_1 - \varphi_2) , \hSigma (\varphi_1 - \varphi_2) \rangle \,.$$
The following theorem provides a finite sample concentration bound on the estimation error: 

\begin{theorem}
\label{thm:groupSVD}
Assume that $\varphi_i$ have subgaussian tail, i.e .$\|\varphi_i\|_{\psi_2} < \infty$. Then with probability $\ge 1 - 2e^{-ct^2}$ we have: 
    $$\textstyle\|\hSigma - \Sigma_0\|_{op} \le b + \frac{\delta \vee \delta^2}{\tilde{\gamma} - (\delta \vee \delta^2)} + \frac{t}{n}$$ for all $t<(\sqrt{n}\tilde{\gamma} - C\sqrt{d}) \wedge (\sqrt{n\tilde{\gamma}} - C\sqrt{d})$ where: 
\begin{enumerate}
% \item  $\tilde{A}_*$ is the collection of the eigenvectors of $A_*$. 
    \item $b = \left(\frac{\lambda_{min}(A_*A_*^T)}{\|B_*B_*^T + \sigma^2 I_d\|_{op}}-1\right)^{-1}$
 
    \item $\delta = \frac{C\sqrt{d} + t}{\sqrt{n}}$. 
    \item $\tilde{\gamma} = \lambda_{min}(A_*A_*^T) - \|B_* B_*^T\|_{op}$.
\end{enumerate}{}
The constants $C,c$ only depend on the subgaussian norm constant of $\phi_i$. 
\end{theorem}

The error bound provided by Theorem \ref{thm:groupSVD} is similar to the error bound provided by Theorem \ref{thm:svd_theorem_new} consists of two terms. The first term $\bar{B}$ is again the approximation error/bias in the estimate of the sensitive subspace due to heterogeneity in the group; it has the same form as the bias as in Theorem \ref{thm:svd_theorem_new} and has a similar interpretation. The second term is the estimation error, which is also similar to the estimation error term in Theorem \ref{thm:svd_theorem_new}. The third term is the error incurred in estimating the mean of the $\varphi_i$'s. It is a higher order term and does not affect the rate of convergence of the estimator. 
% As before our estimated fair metric will be: 
% $$$$
% where $\hat \Sigma = \left(I - \hat U \hat U^{\top}\right)$.

\subsection{Training individually fair ML models with FACE and SenSR}
We envision FACE as the first stage in a pipeline for training fair ML models. In this section, we show that FACE in conjunction with SenSR \cite{yurochkin2020Training} trains individually fair ML models. To keep things concise, we adopt the notation of \cite{yurochkin2020Training}. We start by stating our assumptions on the ML task.

\begin{enumerate}
\item We assume the embeded feature space of $\varphi$ is bounded $R\triangleq\max\{\diam(\varphi),\diam_*(\varphi)\} < \infty$, where $\diam_*$ is the diameter of $\varphi$ in the (unknown) exact fair metric 
\[
d_x^*(x_1,x_2) = \langle (\varphi_1 - \varphi_2) , \Sigma_0 (\varphi_1 - \varphi_2) \rangle^{1/2},
\]
and $\diam$ is the diameter in the learned fair metric
\[
d_x(x_1,x_2) = \langle (\varphi_1 - \varphi_2) , \hSigma (\varphi_1 - \varphi_2) \rangle^{1/2}
.
\]
\item Define $\mathcal{L} = \{\ell(\cdot, \theta) \ : \ \theta \in \Theta \}$ as the loss class. We assume the functions in the loss class $\cL = \{\ell(\cdot,\theta):\theta\in\Theta\}$ are non-negative and bounded: $0 \le \ell(z,\theta) \le M$ for all $z\in\cZ$ and $\theta\in\Theta$, and $L$-Lipschitz with respect to $d_x$: 
\item the discrepancy in the fair metric is uniformly bounded: there is $\delta_c > 0$ such that
\[\textstyle
\begin{aligned}\textstyle
&\sup_{(x_1, x_2)\in\cZ}|d_x^2(x_1, x_2) - (d_x^*(x_1, x_2))^2| \le \delta_cR^2.
\end{aligned}
\]
\end{enumerate}

The third assumption is satisfied with high probability as long as $\delta_c \ge (b + \frac{\delta \vee \delta^2}{\tilde{\gamma} - (\delta \vee \delta^2)})$. 

\begin{theorem}
\label{thm:provably-fair-training}
Under the preceding assumptions, if we define $\delta^* \ge 0$ such that: 
\begin{equation}\textstyle
\min_{\theta \in \Theta} \sup_{P:W_*(P,P_*) \le \eps}\Ex_P\big[\ell(Z,\theta)\big] =  \delta^*
\label{eq:fairClassifierExists}
\end{equation}
and
\[\textstyle
\htheta\in\argmin_{\theta\in\Theta}\sup_{P:W(P,P_n) \le \eps}\Ex_P\big[\ell(Z,h)\big] \,,
\]
then the estimator $\hat \theta$ satisfies:
\begin{equation}\textstyle
\sup_{P:W_*(P,P_*) \le \eps}\Ex_P\big[\ell(Z,\htheta)\big] - \Ex_{P_*}\big[\ell(Z,\htheta)\big] \le \delta^* + 2\delta_n,
\label{eq:fair-gap}
\end{equation}
where $W$ and $W_*$ are the learned and exact fair Wasserstein distances induced by the learned and exact fair metrics (see Section 2.1 in \citet{yurochkin2020Training}) and
\[\textstyle
\delta_n \le \frac{48\mathfrak{C}(\cL)}{\sqrt{n}} + \frac{48LR^2}{\sqrt{n\eps}} + \frac{L\delta_cR^2}{\sqrt{\eps}} + M\left(\frac{\log\frac2t}{2n}\right)^{\frac12}.
\]
where $\mathfrak{C}(\cL) = \int_0^{\infty} \sqrt{\log{\left(\mathcal{N}_{\infty}\left(\cL, r\right)\right)}} \ dr$,  with $\mathcal{N}_{\infty}\left(\cL, r\right)$ being the covering number of the loss class $\mathcal{L}$ with respect to the uniform metric. 
\end{theorem}

Theorem \ref{thm:provably-fair-training} guarantees FACE in conjunction with SenSR trains an individually fair ML model in the sense that its fair gap \eqref{eq:fair-gap} is small. Intuitively, a small fair gap means it is not possible for an auditor to affect the performance of the ML model by perturbing the training examples in certain ``sensitive'' ways.

% Along the same line of discussion at the end of Theorem \ref{concentration} we can combine Theorem \ref{thm:svd_theorem_new} (or Theorem \ref{thm:groupSVD}) to Proposition 3.1 of \cite{yurochkin2020Training} to get the following concentration bound on the excess risk: 
% \begin{corollary}
% Using the above definitions of $\delta_n$ and assuming the loss function $\ell \in \mathcal{L}$ and $\|z_i\|_2 \le R$, we have with probability $\ge 1-t-e^{-bt^2}$: 
% \begin{align*}
% \delta_n &\textstyle \le \frac{48\mathcal{C}(\mathcal{L})}{\sqrt{n}} + \frac{48LR^2}{\sqrt{n}\epsilon} \\
% &\textstyle \qquad \qquad + \frac{LR^2t}{\sqrt{n}\lambda_{\min}(\Sigma_0)\sqrt{\epsilon}}+M\left(\frac{\log{(2/t)}}{2n}\right)^{1/2}
% \end{align*}
% \end{corollary}
% \yk{separate this result into two theorems, one at the population level, and one showing the finite-sample estimation error}

The same conclusion can also be drawn using Theorem \ref{thm:groupSVD} with essentially similar line of arguments.

\begin{remark}
The theory of EXPLORE is same in spirit with the theory of Face. In EXPLORE, we try to learn fair metric from comparable and incomparable pairs. As mentioned in the previous section, we solve MLE under the assumption of quadratic logit link to estimate $\Sigma_0$. Under the assumption that the parameter space and the space of  embedded covariates ($\varphi(x)$) are boudned, we can establish the finite sample concentration bound of our estimator. It is also possible to combine our results with the results of \citet{yurochkin2020Training} to obtain guarantees on the individual fairness of ML models trained with EXPLORE and SenSR (see Corollary \ref{cor:EXPLORE+SENSR}).  
\end{remark}

% \addtolength{\tabcolsep}{-3pt}
% \begin{table*}[t]
% \centering
% \caption{Association tests code names}
% \vspace{.05in}
% \begin{tabular}{lc}
% \toprule
% FLvINS & Flowers vs. insects \citep{greenwald1998measuring}\\
% INSTvWP & Instruments vs. weapons \citep{greenwald1998measuring}\\
% MNTvPHS & Mental vs. physical disease \citep{monteith2011implicit}\\
% EAvAA & Europ-Amer vs Afr-Amer names \citep{caliskan2017Semantics}\\
% EAvAA\citep{bertrand2004Are} & Europ-Amer vs Afr-Amer names \citep{bertrand2004Are}\\
% MNvFN & Male vs. female names \citep{nosek2002harvesting}\\
% MTHvART & Math vs. arts \citep{nosek2002harvesting}\\
% SCvART\citep{nosek2002math} & Science vs. arts \citep{nosek2002math}\\
% YNGvOLD & Young vs. old people's names \citep{nosek2002harvesting}\\
% \midrule
% PLvUPL &  Pleasant vs. unpleasant \citep{greenwald1998measuring} \\
% TMPvPRM & Temporary vs. permanent \citep{monteith2011implicit} \\
% PLvUPL\citep{nosek2002harvesting} &  Pleasant vs. unpleasant \citep{nosek2002harvesting} \\
% CARvFAM & Career vs. family \citep{nosek2002harvesting}\\
% MTvFT & Male vs. female terms \citep{nosek2002harvesting}\\
% MTvFT\citep{nosek2002math} & Male vs. female terms \citep{nosek2002math}\\
% \bottomrule
% \end{tabular}
% \label{table:association_names}
% \end{table*}

% \addtolength{\tabcolsep}{+3pt}

\addtolength{\tabcolsep}{-3pt}
 \captionsetup{labelsep=newline}
\begin{table*}[t]
\centering
% \captionsetup{justification=centering}

\caption{Word Embedding Association Test (WEAT) results. $p$-values that are  significant/insignificant at the 0.05-level are shown in bold. See Table \ref{table:association_names} in Appendix for the unabbreviated forms of the targets and attributes}
\vspace{.05in}
\begin{tabular}{cccccccccccccccc}
\toprule
\multirow{2}{*}{Target} & \multirow{2}{*}{Attribute} & \multicolumn{2}{c}{Euclidean} & \multicolumn{2}{c}{} & \multicolumn{2}{c}{} & \multicolumn{2}{c}{EXPLORE} & \multicolumn{2}{c}{FACE-3} & \multicolumn{2}{c}{FACE-10} & \multicolumn{2}{c}{FACE-50} \\
% \toprule
      &   &  \textbf{P} & \textbf{d} &  \textbf{P} & \textbf{d} &  \textbf{P} & \textbf{d} &  \textbf{P} & \textbf{d} &  \textbf{P} & \textbf{d} &  \textbf{P} & \textbf{d} &  \textbf{P} & \textbf{d} \\
\midrule
   FLvINS &   PLvUPL &        \textbf{0.00} &       1.58 &        \textbf{0.00} &       1.55 &  \textbf{0.00} &       1.55&      \textbf{0.00} &       1.41 &        \textbf{0.00} &       1.59 &        \textbf{0.00} &       1.56 &        \textbf{0.00} &       1.27 \\
  INSTvWP &   PLvUPL &        \textbf{0.00} &       1.46 &        \textbf{0.00} &       1.45 &   \textbf{0.00} &       1.46&     \textbf{0.00} &       1.44 &        \textbf{0.00} &       1.48 &        \textbf{0.00} &       1.58 &        \textbf{0.00} &       1.49 \\
  MNTvPHS &   TMPvPRM &        \textbf{4e-5} &       1.54 &        \textbf{4e-5} &       1.54 &    \textbf{4e-5} &       1.54&    \textbf{4e-4} &       1.31 &        \textbf{4e-5} &       1.56 &        \textbf{0.00} &        1.6 &        \textbf{0.00} &       1.68 \\
\midrule
    EAvAA &   PLvUPL &        0.00 &       1.36 &        0.00 &       1.36 &   0.00 &       1.38 &     1e-2 &       0.62 &        \textbf{5e-1} &       0.17 &        7e-2 &       0.46 &        \textbf{2e-1} &       0.33 \\
   EAvAA  &   PLvUPL &        0.00 &       1.49 &        0.00 &       1.51 &    0.00 &       1.51&    \textbf{2e-1} &       0.49 &        \textbf{7e-1} &       0.15 &        \textbf{6e-2} &       0.67 &        \textbf{2e-1} &       0.51 \\
   EAvAA  &  PLvUPL  &        8e-5 &       1.31 &        4e-5 &       1.41 &    4e-5 &       1.41 &    \textbf{1e-1} &       0.55 &        \textbf{4e-1} &       0.31 &        \textbf{4e-1} &        0.3 &        \textbf{4e-1} &       0.34 \\
    MNvFN &   CARvFAM &        0.00 &       1.69 &        2e-3 &       1.23 &    2e-3 &       1.23 &    \textbf{2e-1} &       0.25 &        1e-3 &       1.24 &        6e-3 &       1.13 &        \textbf{8e-2} &       0.53 \\
  MTHvART &     MTvFT &        8e-5 &        1.5 &        3e-2 &       0.84 &   1e-3 &       1.34 &     1e-3 &       1.34 &        1e-3 &       1.35 &        4e-3 &       1.18 &        6e-3 &       1.16 \\
 SCvART  &   MTvFT  &        9e-3 &       1.05 &        9e-3 &       1.08 &    4e-2 &       0.76 &    \textbf{6e-2} &       0.65 &        4e-2 &       0.72 &        3e-2 &       0.84 &        \textbf{1e-1} &        0.3 \\
  YNGvOLD &  PLvUPL  &        1e-2 &        1.0 &        2e-4 &        1.5 &   1e-4 &       1.53 &     \textbf{7e-2} &        0.6 &        \textbf{9e-2} &        0.5 &        2e-3 &       1.27 &        4e-3 &       1.16 \\
\bottomrule
\end{tabular}
\label{table:association_result}
\end{table*}
\addtolength{\tabcolsep}{+3pt}

\section{Computational results}
\label{sec:computationalResults}

In this section, we investigate the performance of the learned metrics on two ML tasks: income classification and sentiment analysis.

\subsection{Eliminating biased word embeddings associations}
Many recent works have observed biases in word embeddings \citep{bolukbasi2016Man, caliskan2017Semantics, brunet2019understanding, dev2019attenuating, zhao2019gender}. \citet{bolukbasi2016Man}
studied gender biases through the task of finding analogies and proposed a popular debiasing algorithm. \citet{caliskan2017Semantics} proposed a more methodological way of analyzing various biases through a series of Word Embedding Association Tests (WEATs). \emph{We show that replacing the metric on the word embedding space with a fair metric learned by FACE or EXPLORE eliminates most biases in word embeddings}.
% Through a series of WEATs, we investigate whether the debiasing methods \citep{bolukbasi2016Man,dev2019attenuating} and our metric learning approaches can eliminate stereotypical associations pertaining to the widely used GloVe embeddings \cite{pennington2014glove}, while maintaining meaningful ones crucial for success in many NLP tasks.

\paragraph{Word embedding association test}
Word embedding association test (WEAT) was developed by \cite{caliskan2017Semantics} to evaluate semantic biases in word embeddings. The tests are inspired by implicit association tests (IAT) from the psychometrics literature \cite{greenwald1998measuring}. Let $\cX,\cY$ be two sets of word embeddings of \emph{target words} of equal size (\eg\ African-American and European-American names respectively), and $\cA,\cB$ be two sets of \emph{attribute words} (\eg\ words with positive and negative sentiment respectively). For each word $x\in\cX$, we measure its association with the attribute by \begin{equation}
\label{eq:association}
  s(x,\cA,\cB) \triangleq \frac{1}{|\cA|}\sum_{a\in\cA}\frac{\langle x,a\rangle}{\|x\|\|a\|} - \frac{1}{|\cB|}\sum_{b\in\cB}\frac{\langle x,b\rangle}{\|x\|\|b\|}  
\end{equation}
If $x$ tends to be associated with the attribute (\eg\ it has positive or negative sentiment), then we expect $s(x,A,B)$ to be far from zero. To measure the association of $\cX$ with the attribute, we average the associations of the words in $\cX$:
\[
s(\cX,\cA,\cB) \triangleq \frac{1}{|\cX|}\sum_{x\in\cX}s(x,\cA,\cB).
\]
Following \cite{caliskan2017Semantics}, we consider the absolute difference between the associations of $\cX$ and $\cY$ with the attribute as a test statistic:
\[
s(\cX,\cY,\cA,\cB) \triangleq |s(\cX,\cA,\cB) - s(\cY,\cA,\cB)|.
\]
Under the null hypothesis, $\cX$ and $\cY$ are equally associated with the attribute (\eg\ names common among different races have similar sentiment). This suggests we calibrate the test by permutation. Let $\{(X_\sigma,Y_\sigma)\}_\sigma$ be the set of all partitions of $\cX\cup\cY$ into two sets of equal size. Under the null hypothesis, $s(\cX,\cY,\cA,\cB)$ should be typical among the values of $\{s(\cX_\sigma,\cY_\sigma,\cA,\cB)\}$. We summarize the ``atypicality'' of $s(\cX,\cY,\cA,\cB)$ with a two-sided $p$-value\footnote{\citet{caliskan2017Semantics} used one-sided $p$-value, however we believe that inverse association is also undesired and use a two-sided one}
\[
\mathbf{P} = \frac{\sum_\sigma\ones\{s(\cX_\sigma,\cY_\sigma,\cA,\cB) > s(\cX,\cY,\cA,\cB)\}}{\card(\{(X_\sigma,Y_\sigma)\}_\sigma)}.
\]
Following \cite{caliskan2017Semantics}, we also report a standardized effect size
\[
\mathbf{d} = \frac{s(\cX,\cY,\cA,\cB)}{\stdev(\{s(x,\cA,\cB)\}_{x\in\cX\cup\cY})}
\]
for a more fine-grained comparison of the methods.

\paragraph{Learning EXPLORE and FACE:}

To apply our fair metric learning approaches we should define a set of comparable samples for FACE and a collection of comparable and incomparable pairs for EXPLORE.

For the set of comparable samples for FACE we choose embeddings of a side dataset of 1200 popular baby names in New York City\footnote{available from \url{https://catalog.data.gov/dataset/}}. The motivation is two-fold: (i) from the perspective of individual fairness, it is reasonable to say that human names should be treated similarly in NLP tasks such as resume screening; (ii) multiple prior works have observed that names capture biases in word embeddings and used them to improve fairness in classification tasks \citep{romanov2019What, yurochkin2020Training}. We consider three choices for the number of factors of FACE: 3, 10 and 50.
% (iii) some of the WEATs \citep{caliskan2017Semantics} consider names as target words and it is of interest of a side dataset of names will ``generalize''.

For EXPLORE we construct comparable pairs by sampling pairs of names from the same pool of popular baby names, however because there are too many unique pairs, we subsample a random 50k of them. To generate the incomparable pairs we consider random 50k pairs of positive and negative words sampled from the dataset proposed by \citet{hu2004Mining} for the task of sentiment classification.

\paragraph{WEAT results} First we clarify how the associations \eqref{eq:association} are computed for different methods. The Euclidean approach is to use word embeddings and directly compute associations in the vanilla Euclidean space; the approaches of \citet{bolukbasi2016Man}
\footnote{\url{https://github.com/tolga-b/debiaswe}}
and \citet{dev2019attenuating}
\footnote{\url{github.com/sunipa/Attenuating-Bias-in-Word-Vec}}
is debias word embeddings before computing associations; associations with FACE and EXPLORE are computed in the Mahalanobis metric space parametrized by a corresponding $\Sigma$, i.e. the inner product $\langle x, y \rangle = x^T \Sigma y$ and norm $\|x\| = \sqrt{\langle x,\Sigma x\rangle}$. When computing \textbf{P}, if the number of partitions of target words $\card(\{(X_\sigma,Y_\sigma)\}_\sigma)$ is too big, we subsample 50k partitions.

We evaluate all of the WEATs considered in \citep{caliskan2017Semantics} with the \emph{exact same} target and attribute word combinations. The results are presented in Table \ref{table:association_result}.

First we verify that all of the methods preserve the celebrated ability of word embeddings to represent semantic contexts --- all WEATs in the upper part of the table correspond to meaningful associations such as Flowers vs Insects and Pleasant vs Unpleasant and all $p$-values are small corresponding to the significance of the associations.

On the contrary, WEATs in the lower part correspond to racist (European-American vs African-American names and Pleasant vs Unpleasant) and sexist (Male vs Female names and Career vs Family) associations. The presence of such associations may lead to biases in AI systems utilizing word embeddings. Here, larger $p$-value \textbf{P} and smaller effect size \textbf{d} are desired. We see that previously proposed debiasing methods \citep{bolukbasi2016Man,dev2019attenuating}, although reducing the effect size mildly, are not strong enough to statistically reject the association hypothesis. Our fair metric learning approaches EXPLORE and FACE (with 50 factors) each successfully removes 5 out of 7 unfair associations, including ones not related to names. We note that there is one case, Math vs Arts and Male vs Female terms, where all of our approaches failed to remove the association. We think that, in addition to names, considering a group of comparable gender related terms for FACE and comparable gender related pairs for EXPLORE can help remove this association. 
% \mik{mention that larger number of factors is not always helpful; bold smallest effect size?}

When comparing FACE to EXPLORE, while both performed equally well on the WEATs, we note that learning fair metric using human names appears more natural with FACE. We believe that \emph{all} names are comparable and any major variation among their embeddings could permeate bias in all of the word embeddings. FACE is also easier to implement and utilize than EXPLORE, as it is simply a truncated SVD of the matrix of names embeddings.

\begin{table*}
\caption{Summary of \textbf{Adult} experiment over 10 restarts. Results for all prior methods are copied from \citet{yurochkin2020Training}}
\label{table:adult}
\vskip 0.1in
\begin{center}
\begin{tabular}{lcccccccc}
% \multicolumn{1}{c}{\bf PART}  &\multicolumn{1}{c}{\bf DESCRIPTION}
\toprule
{} &        B-Acc,\% &   $\text{S-Con.}$ &  $\text{GR-Con.}$ &      $\mathrm{Gap}_G^{\mathrm{RMS}}$ &       $\mathrm{Gap}_R^{\mathrm{RMS}}$ &     $\mathrm{Gap}_G^{\mathrm{max}}$ & $\mathrm{Gap}_R^{\mathrm{max}}$ \\
\midrule
SenSR+Explore (With gender) & 79.4 & \textbf{0.966} & 0.987 & \textbf{0.065} & \textbf{0.044} & \textbf{0.084} & \textbf{0.059} \\
SenSR+Explore (Without gender) & 78.9 & 0.933 & 0.993 & 0.066 & 0.05 & \textbf{0.084} & 0.063
\\
SenSR &  78.9 & .934 & .984 & .068 & .055 & .087 & .067 \\
Baseline &  \textbf{82.9} & .848 & .865 & .179 & .089 & .216 & .105 \\
Project & 82.7 & .868 & \textbf{1.00} & .145 & .064 & .192 & .086 \\
Adv. debiasing &  81.5 & .807 & .841 & .082 & .070 & .110 & .078  \\
% Sinha et al & 0 & 0 & 0 & 0 & 0 & 0 & 0 & 0 & 0 \\
CoCL &  79.0 & - & - &  .163 & .080 & .201  & .109  \\
\bottomrule
\end{tabular}
\end{center}
\vskip -0.1in
\end{table*}

\subsection{Applying EXPLORE with SenSR}
SenSR is a method for training fair ML system given a fair metric  \cite{yurochkin2020Training}. In this paper we apply SenSR along with the fair metric learned using EXPLORE on the adult dataset \cite{bache2013UCI}. This data-set consists of 14 attributes over 48842 individuals. The goal is to predict whether each individual has income more than $50$k or not based on these attributes. For applying EXPLORE, we need comparable and incomparable pairs. We define two individuals to be comparable if they belong to the same income group (i.e. both of them has $>50$k or $<50$k annual salary) but with opposite gender, whereas two individuals are said to be incomparable if they belong to the different income group. Based on this labeling, we learn fair metric $\hat{\Sigma}$ via EXPLORE. Finally, following \citet{yurochkin2020Training}, we project out a ``sensitive subspace'' defined by the coefficients of a logistic regression predicting gender from $\hat \Sigma$ i.e.: 
$$\hat \Sigma \longleftarrow (I - P_{gender}) \hat \Sigma (I - P_{gender}) \,.$$
where $P_{gender}$ is the projection matrix on the span of this sensitive subspace. We then apply SenSR along with
$$d_x(x_1, x_2) = (x_1 - x_2)^{\top}\hat \Sigma (x_1 - x_2) \,.$$

Although most of the existing methods use protected attribute to learn a fair classifier, this is not ideal as in many scenarios protected attributes of the individuals are not known. So, it is advisable to learn fair metric without using the information of protected attributes. In this paper we learned our metrics in two different ways (with or without using protected attribute) for comparison purpose: 
\begin{enumerate}
    \item \textbf{SenSR + EXPLORE (with gender)} utilizes gender attribute in classification following prior approaches.
    \item \textbf{SenSR + EXPLORE (without gender)} discards gender when doing classification.
\end{enumerate}
In \citet{yurochkin2020Training}, the authors provided a comparative study of the individual fairness on Adult data. They considered balanced accuracy (B-Acc) instead of accuracy due to class imbalance. The other metrics they considered for performance evaluations are prediction consistency of the classifier with respect to marital status (S-Con., i.e. spouse consistency) and with respect to sensitive attributes like race and gender (GR-Con.). They also used RMS gaps and maximum gaps between true positive rates across genders ($\mathrm{Gap}_G^{\mathrm{RMS}}$ and $\mathrm{Gap}_G^{\mathrm{max}}$) and races ($\mathrm{Gap}_R^{\mathrm{RMS}}$ and $\mathrm{Gap}_R^{\mathrm{max}}$) for the assessment of group fairness (See Appendix for the detailed definition). Here we use their results and compare with our proposed methods. The results are summarized in Table \ref{table:adult}. It is evident that SenSR + EXPLORE (both with gender and without gender) outperforms SenSR (propsoed in \cite{yurochkin2020Training}) in almost every aspect. Discarding gender in our approach prevents from violations of individual fairness when flipping the gender attribute as seen by improved gender and race consistency metric, however accuracy, spouse consistency and group fairness metrics are better when keeping the gender. Despite this we believe that it is better to avoid using gender in income classification as it is highly prone to introducing unnecessary biases.

\section{Summary and discussion}

We studied two methods of learning the fair metric in the definition of individual fairness and showed that both are effective in ignoring implicit biases in word embeddings. Our methods remove one of the main barriers to wider adoption of individual fairness in machine learning. We emphasize that our methods are probabilistic in nature and naturally robust to inconsistencies in the data. 
Together with tools for training individually fair ML models \cite{yurochkin2020Training}, the methods presented here complete a pipeline for ensuring that ML models are free from algorithmic bias/unfairness.

% \subsubsection*{Acknowledgements}

% Use the unnumbered third level heading for the acknowledgements.  All
% acknowledgements go at the end of the paper.

% \newpage
\bibliography{goku,roxy}
\bibliographystyle{icml2020}

\onecolumn
\appendix
\addtolength{\tabcolsep}{-3pt}
\begin{table*}[t]
\centering
\caption{Association tests code names}
\vspace{.05in}
\begin{tabular}{lc}
\toprule
FLvINS & Flowers vs. insects \citep{greenwald1998measuring}\\
INSTvWP & Instruments vs. weapons \citep{greenwald1998measuring}\\
MNTvPHS & Mental vs. physical disease \citep{monteith2011implicit}\\
EAvAA & Europ-Amer vs Afr-Amer names \citep{caliskan2017Semantics}\\
EAvAA\citep{bertrand2004Are} & Europ-Amer vs Afr-Amer names \citep{bertrand2004Are}\\
MNvFN & Male vs. female names \citep{nosek2002harvesting}\\
MTHvART & Math vs. arts \citep{nosek2002harvesting}\\
SCvART\citep{nosek2002math} & Science vs. arts \citep{nosek2002math}\\
YNGvOLD & Young vs. old people's names \citep{nosek2002harvesting}\\
\midrule
PLvUPL &  Pleasant vs. unpleasant \citep{greenwald1998measuring} \\
TMPvPRM & Temporary vs. permanent \citep{monteith2011implicit} \\
PLvUPL\citep{nosek2002harvesting} &  Pleasant vs. unpleasant \citep{nosek2002harvesting} \\
CARvFAM & Career vs. family \citep{nosek2002harvesting}\\
MTvFT & Male vs. female terms \citep{nosek2002harvesting}\\
MTvFT\citep{nosek2002math} & Male vs. female terms \citep{nosek2002math}\\
\bottomrule
\end{tabular}
\label{table:association_names}
\end{table*}

\newpage

\section{Relation between groupwise and pairwise comparison}
In case of pairwise comparison, we have $|I_1| = \dots = |I_G| = 2$. As mentioned in the Algorithm \ref{alg:face}, we at first mean-center each group, which is assumed to nullify the variability along the directions of the relevant attributes. Lets consider $I_1 = \{\varphi_{1_1}, \varphi_{1_2}\}$. Then: 
\begin{align*}
    H\Phi_{I_1} & = \left(\varphi_{1_1} - \frac{\varphi_{1_1} + \varphi_{1_2} }{2}, \varphi_{1_2} - \frac{\varphi_{1_1} + \varphi_{1_2}}{2} \right)^{\top} \\
    & = \left(\frac{\varphi_{1_1} - \varphi_{1_2} }{2},  \frac{\varphi_{1_2} - \varphi_{1_1}}{2} \right)^{\top}
\end{align*}
Hence the combined matrix can be written as: 
$$M_{pairs} = \frac{1}{4|G|}\sum_{i=1}^{|G|}\left(\varphi_{i_1} - \varphi_{i_2}\right)\left(\varphi_{i_1} - \varphi_{i_2}\right)^{\top}$$
which is equivalent to consider the difference between the pairs of each individual groups (upto a constant).  On the other hand, we have more than two observations in each group, the grand matrix following Algorithm \ref{alg:face} becomes: 
$$M_{general} = \frac{1}{N}\sum_{i=1}^{G}\sum_{j=1}^{|I_G|}\left(\varphi_{i_j} - \bar \varphi_{i} \right)\left(\varphi_{i_j} - \bar \varphi_{i}\right)^{\top}$$
where $N = \sum_{i=1}^{G} |I_G|$, total number of observations. Hence, in case of $|I_G| = 2$, we essentially don't need to mean center as we are taking the difference between the observations of each pair. When $G$ is essentially fixed, i.e. $|I_G| \approx N$, the error in estimating $\ran{A_*}$ due to mean centering contributes a higher order term (See Theorem \ref{thm:groupSVD} for more details) which is essentially negligible. In case of pairwise comparison, although there is no error due to mean centering, we pay a constant as we are effectively loosing one observation in the each pair. 
\section{Theoretical properties of EXPLORE}
\label{sec:logisticTheory}
In this section, we investigate the theoretical properties of EXPLORE. We provide statistical guarantees corresponding to the estimation using the scaled logistic link (Section \ref{matrix-logistic}). To keep things simple, we tweak \eqref{eq:binaryResponseModel} so that it is strongly identifiable:
\[
y_i\mid z_{i_1},z_{i_2} \sim \Ber((2-\eps)\sigma(-\langle D_i,\Sigma_0\rangle))
\]
for some small $\eps > 0$. The log-likelihood of samples $(x_1,y_1),\dots,(x_n,y_n)$ is 
\allowdisplaybreaks
\begin{align*}
    \ell_n(\Sigma) &\textstyle = \frac1n \sum_{i=1}^n \left[y_i \log{F_*(x_i'\Sigma x_i)} \right. \\
    & \qquad \qquad \qquad \left. + (1-y_i) \log{(1 - F_*(x_i'\Sigma x_i))}\right],
\end{align*}
where $F_* = (2-\epsilon)\sigma$.

\begin{proposition}
\label{convexity}
The population version of the likelihood function $\ell(\Sigma)$ is concave in $\Sigma$ and uniquely maximized at $\Sigma_0$. 
\end{proposition}
\begin{proof}
The population version of the likelihood function is: 
\allowdisplaybreaks
\begin{align}
    \label{eqn1} \ell(\Sigma) = E\left[Y \log{F_*(X'\Sigma X)} + (1-Y) \log{(1 - F_*(X'\Sigma X))}\right] = g(X'\Sigma X)
\end{align}
As the function $\Sigma \longrightarrow X'\Sigma X$ is affine in $\Sigma$, we only need to show that $g$ is concave. From equation \ref{eqn1}, the function $g(.)$ can be define as: $g(t) = y\log{F_*(t)} + (1-y)\log{(1 - F_*(t))}$ on $t\in \mathbb{R}^+$ for any fixed $y \in \{0,1\}$. The function $F_*$ is double differentiable with the derivatives as below: 
\allowdisplaybreaks
\begin{align*}
    & F_*(x) = \frac{2-\epsilon}{1+e^t}, \ \ 1-F_*(t) = \frac{e^t - 1 + \epsilon}{1+e^t} \\
    & F'_*(t) = -(2-\epsilon)\frac{e^t}{(1+e^t)^2} \\
    & F''_*(t) = -(2-\epsilon)\frac{e^t(1-e^t)}{(1+e^t)^3}
\end{align*}

We show below that $g''(t) \le 0$ for all $t$ which proves the concavity of $\ell(\Sigma)$:
\allowdisplaybreaks
\begin{align}
    & g(t) = y\log{F_*(t)} + (1-y)\log{(1 - F_*(t))} \notag\\
   \Rightarrow &  g'(t) = y\frac{F'_*(t)}{F_*(t)} - (1-y)\frac{F'_*(t)}{1-F_*(t)} \notag\\
   \label{eqn2} \Rightarrow & g''(t) = y\frac{F_*(t)F''_*(t) - (F'_*(t))^2}{F^2_*(t)} - (1-y)\frac{(1-F_*(t))F''_*(t) + (F'_*(t))^2}{(1-F_*(t))^2}
\end{align}
For the first summand in the double derivative we have: 
\allowdisplaybreaks
\begin{align}
    \frac{F_*(t)F''_*(t) - (F'_*(t))^2}{F^2_*(t)} & = \frac{-(2-\epsilon)^2 \frac{e^t(1-e^t)}{(1+e^t)^4} - (2-\epsilon)^2\frac{e^{2t}}{(1+e^t)^4}}{\frac{(2-\epsilon)^2}{(1+e^t)^4}} \notag\\
    \label{eqn3} & = - \frac{e^t(1-e^t) + e^{2t}}{(1+e^t)^2} = -\frac{e^t}{(1+e^t)^2} < 0 \ \forall \ t \in \mathbb{R}^+
\end{align}
For the second summand: 
\allowdisplaybreaks
\begin{align}
    \frac{(1-F_*(t))F''_*(t) + (F'_*(t))^2}{(1-F_*(t))^2} & =
    \frac{-(2-\epsilon)\frac{(e^t - 1 + \epsilon)e^t(1-e^t)}{(1+e^t)^4} + (2-\epsilon)^2\frac{e^{2t}}{(1+e^t)^4}}{\frac{(e^t-1+\epsilon)^2}{(1+e^t)^2}} \notag\\
    & = \frac{(2-\epsilon)\left[(2-\epsilon)e^{2t} - (e^t-1+\epsilon)e^t(1-e^t)\right]}{(e^t -1+\epsilon)^2(1+e^t)^2} \notag\\
    \label{eqn4} & = \frac{(2-\epsilon)\left[(2-\epsilon)e^{2t} + (e^t-1+\epsilon)e^t(e^t-1)\right]}{(e^t -1+\epsilon)^2(1+e^t)^2} \ge 0 \ \forall \ t \in \mathbb{R}^+
\end{align}
Combining equations \ref{eqn2}, \ref{eqn3} and \ref{eqn4} we get: 
\allowdisplaybreaks
\begin{align*}
    g''(t) = -y\frac{e^t}{(1+e^t)^2} -(1-y) \frac{(2-\epsilon)\left[(2-\epsilon)e^{2t} + (e^t-1+\epsilon)e^t(e^t-1)\right]}{(e^t -1+\epsilon)^2(1+e^t)^2}  < 0 \ \forall \ t \in \mathbb{R}^+
\end{align*}
This proves the strict concavity. To prove that $\Sigma_0$ is the unique maximizer, observe that: 
\allowdisplaybreaks
\begin{align*}
    \ell(\Sigma) & = E\left[Y \log{F_*(X'\Sigma X)} + (1-Y) \log{(1 - F_*(X'\Sigma X))}\right] \\
    & = E\left[F_*(X'\Sigma_0 X) \log{F_*(X'\Sigma X)} + (1-F_*(X'\Sigma_0 X)) \log{(1 - F_*(X'\Sigma X))}\right] \\
    & = \ell(\Sigma_0) - E\left(KL(Bern(F_*(X'\Sigma_0 X)) \ || \ Bern(F_*(X'\Sigma X)))\right)
\end{align*}
Hence $\ell(\Sigma_0) \ge \ell(\Sigma)$ for all $\Sigma \in \Theta$ as KL divergence is always non-negative. Next, let $\Sigma_1$ be any other maximizer. Then, 
\allowdisplaybreaks
\begin{align*}
    & E\left(KL(Bern(F_*(X'\Sigma_0 X)) \ || \ Bern(F_*(X'\Sigma X)))\right) = 0 \\
    \Rightarrow & KL(Bern(F_*(X'\Sigma_0 X)) \ || \ Bern(F_*(X'\Sigma X))) = 0 \ \text{a.s. in } \ X \\
    \Rightarrow & F_*(X'\Sigma_0 X) = F_*(X'\Sigma X) \text{a.s. in } \ X \\
    \Rightarrow & X'(\Sigma-\Sigma_0) X = 0  \ \text{a.s. in } \ X \\
    \Rightarrow & \Sigma = \Sigma_0
\end{align*}
as the interior of the support of $X$ is non null. This proves the uniqueness of the maximizer.
\end{proof}

The maximum likelihood estimator (MLE) $\hSigma$ is $$\hSigma = \argmax_{\Sigma} \ell_n(\Sigma)$$ 
The asymptotic properties of $\hSigma$ (consistency and asymptotic normality) are well-established in the statistical literature (e.g. see \citet{vandervaart1998Asymptotic}). Here we study the non-asymptotic convergence rate of the MLE. We start by stating our assumptions.

\begin{assumption}
\label{Sample_space}
The feature space $\mathcal{X}$ is a bounded subset of $\mathbb{R}^d$, \ie\ there exits $R < \infty$ such that $\|X\| = \|\varphi_1 - \varphi_2\| \le U$ for all $X \in \mathcal{X}$. 
\end{assumption}

\begin{assumption}
\label{parameter_space}
The parameter space $\Theta$ is a subset of $\mathbb{S}_{++}^d$ and $\sup\{\lambda_{max}(\Sigma):\Sigma\in\Theta\} \le C_+ < \infty$.
\end{assumption}

\noindent
Under these assumptions, we establish a finite sample concentration result for our estimator $\hSigma$:

\begin{theorem}
\label{concentration}
Under assumptions \ref{Sample_space} and \ref{parameter_space} we have the following $$\sqrt{n} \|\hSigma - \Sigma_0\|_{op} \le t $$ with probability atleast $1-e^{-bt^2}$ for some constant $b > 0$. 
\end{theorem}
\begin{proof}
We break the proof of the theorem into a few small lemmas. Consider the collection $\mathcal{G} = \{g_{\Sigma}: \Sigma \in \Theta\}$, where $$g_{\Sigma}(X,Y) = \left[Y \log{F_*(X'\Sigma X)} + (1-Y) \log{\left(1 - F(X'\Sigma X\right)}\right]$$ The problem of estimating $\Sigma_0$ using MLE can be viewed as a risk minimization problem over the collection of functions $\mathcal{G}$, which we are going to exploit later this section. Lemma \ref{curvature_logistic} below provides a lower bound on the deviation of $l(\Sigma)$ from $l(\Sigma_0)$ in terms of $\|\Sigma - \Sigma_0\|_{op}$: 

\begin{lemma}
\label{curvature_logistic}
Under assumptions \ref{Sample_space} and \ref{parameter_space}, we have a quadratic lower bound on the excess risk: $$\ell(\Sigma_0) - \ell(\Sigma) \gtrsim \|\Sigma - \Sigma_0\|^2_{op}$$
\end{lemma}

\begin{proof}
From the definition of our model in ExPLORE, $F_*(t) = (2-\epsilon)/(1+e^t)$ which implies $F'_*(t) = -(2-\epsilon)e^t/(1+e^t)^2$. As $X$ is bounded (Assumption \ref{Sample_space}), $$ \langle XX^T, \Sigma \rangle \le \lambda_{max}(\Sigma) \|X\|_2^2 \le C_+ U^2$$ for all $X \in \mathcal{X}, \Sigma \in \Theta$, where the constants $C_+$ and $U$ are as defined in Assumptions \ref{parameter_space} and \ref{Sample_space} respectively. Hence, there exists $\tilde{K} > 0$ such that $|F'_*(X'(\alpha \Sigma + (1-\alpha)\Sigma_0)X)| \ge \tilde{K}$ for all $X, \Sigma$. For notational simplicity define $D = XX^T$. From the definition of $l(\Sigma)$ we have: 
\allowdisplaybreaks
\begin{align}
l(\Sigma) & = E\left(F_*(\langle D, \Sigma_0 \rangle)\log{F_*(\langle D, \Sigma\rangle)}+(1-F_*(\langle D, \Sigma_0 \rangle))(1-\log{F_*(\langle D, \Sigma\rangle)})\right) \notag\\
& = l(\Sigma_0) - E\left[KL\left(Bern(F_*(\langle D, \Sigma^0 \rangle)) \ || \ Bern(F_*(\langle D, \Sigma \rangle))\right)\right] \notag\\
\label{est1} & \le l(\Sigma^0) - 2E\left[\left(F_*(\langle D, \Sigma_0 \rangle) - F_*(\langle D, \Sigma \rangle)\right)^2\right] 
\end{align}
where the last inequality follows from Pinsker's inequality. Using equation \ref{est1} we can conclude:
\allowdisplaybreaks
\begin{align*}
l(\Sigma^0) - l(\Sigma) & \ge 2E\left[\left(F_*(\langle D, \Sigma_0 \rangle) - F_*(\langle D, \Sigma \rangle)\right)^2\right] \\
& \ge 2\tilde{K}^2E\left[\left(\langle D, \Sigma -  \Sigma^0 \rangle\right)^2\right] \\
& \ge 2\tilde{K}^2\|\Sigma - \Sigma_0\|^2_{op}E\left[\left(\langle D, \frac{\Sigma -  \Sigma_0}{\|\Sigma - \Sigma_0\|_{op}} \rangle\right)^2\right] \\
& \ge 2\tilde{K}^2\|\Sigma - \Sigma_0\|^2_{op} E\left[\left( X^T\frac{\Sigma -  \Sigma_0}{\|\Sigma - \Sigma_0\|_{op}}X\right)^2\right] \\
& \ge 2c\tilde{K}^2\|\Sigma - \Sigma_0\|^2_{op}
\end{align*}
Here we have used the fact that $$\inf_{T \in S_d^{++}: \|T\|_{op} = 1}E\left[\left( X^TTX\right)^2\right] = c > 0$$ To prove the fact, assume on the contrary that the infimum is $0$. The set of all matrices $T$ with $\|T\|_{op} = 1$ is compact subset of $\mathbb{R}^{d \times d}$. Now consider the function: $$f: T \longrightarrow E\left[\left( X^TTX\right)^2\right]$$ By DCT, $f$ is a continuous function. Hence the infimum will be attained, which means that we can find a matrix $M$ such that $M \in S_{d}^{++}$ and $\|M\|_{op} = 1$ such that $E\left[\left( X^TMX\right)^2\right] = 0$. Hence $X^TMX = 0$ almost surely. As the support of $A$ contains an open set, we can conclude $M = 0$ which contradicts $\|M\|_{op} = 1$. 

\end{proof}

\noindent
Next we establish an upper bound on the variability of the centered function $g_{\Sigma} - g_{\Sigma_0}$ in terms of the distance function, which is stated in the following lemma: 

\begin{lemma}
\label{variance-bound}
Under the aforementioned assumptions, $$Var\left(g_{\Sigma} - g_{\Sigma_0}\right) \lesssim d^2(\Sigma, \Sigma_0)$$ where $d(\Sigma, \Sigma_0) = \|\Sigma - \Sigma_0\|_{op}$.
\end{lemma}

\begin{proof}
We start with the observation $$g_{\Sigma_0}(X,Y) - g_{\Sigma}(X,Y) = Y\log{\frac{F_*(X'\Sigma_0 X)}{F_*(X'\Sigma X)}} + (1-Y)\log{\frac{1-F_*(X'\Sigma_0 X)}{1-F_*(X'\Sigma X)}}$$ From our assumption on the parameter space, we know there exists $p > 0$ such that $p \le F_*(X'\Sigma X) \le 1-p$ for all $\Sigma \in \Theta$ and for all $
X$ almost surely. Hence, 
\allowdisplaybreaks 
\begin{align*}
    |g_{\Sigma_0}(X,Y) - g_{\Sigma}(X,Y)| & \le \left|\log{\frac{F_*(X'\Sigma_0 X)}{F_*(X'\Sigma X)}}\right| + \left|\log{\frac{1-F_*(X'\Sigma_0 X)}{1-F_*(X'\Sigma X)}}\right| \\
    & \le 2K |F_*(X'\Sigma X) - F_*(X'\Sigma_0 X)| \hspace{0.2in} [K \ \text{is the upper bound on the derivative of log}] \\
    & \le K |X'(\Sigma - \Sigma_0)X| \hspace{0.2in} [\text{As} \ F'_* \le 1/2] \\
    & \le KU \|\Sigma - \Sigma_0\|_{op}
\end{align*}
This concludes the lemma.
\end{proof}

The following lemma establishes an upper bound on the modulus of continuity of the centered empirical process: 

\begin{lemma}
\label{modulus}
Under the aforementioned assumptions, we have for any $\delta > 0$: $$E\left(\sup_{d(\Sigma, \Sigma_0) \le \delta} \left|\P_n\left(g_{\Sigma_0} - g_{\Sigma}\right) - P\left(g_{\Sigma_0} - g_{\Sigma}\right)\right|\right) \lesssim \delta$$
\end{lemma}
\begin{proof}
Fix $\delta > 0$. Define $\mathcal{H}_{\delta} = \{h_{\Sigma} = g_{\Sigma} - g_{\Sigma_0}: \|\Sigma - \Sigma_0\|_{op} \le \delta\}$. We can write $g_{\Sigma} - g_{\Sigma_0} = h^1_{\Sigma} + h^2_{\Sigma}$ where $$h^{(1)}_{\Sigma} = Y\log{\frac{F_*(X'\Sigma_0 X)}{F_*(X'\Sigma X)}}, \ h^{(2)}_{\Sigma} = (1-Y)\log{\frac{1-F_*(X'\Sigma_0 X)}{1-F_*(X'\Sigma X)}}$$
Hence $H_{\delta} \subset H^{(1)}_{\delta} + H^{(2)}_{\delta}$ where $H^{(i)}_{\delta} = \{h^{(1)}_{\Sigma}: \ \|\Sigma - \Sigma_0\|_{op} \le \delta\}$for $i \in \{1,2\}$. Next, we argue that $H^{(1)}_{\delta}$ has finite VC dimension. To see this, consider the function $(X,Y) \rightarrow \langle XX', \Sigma \rangle$. As this is linear function, it has finite VC dimension. Now the function $\log \circ F_*$ is monotone. As composition of monotone functions keeps VC dimension finite, we see that $(X,Y) \rightarrow \log{F_*(\langle XX', \Sigma \rangle)}$ is also VC class. It is also easy to see that projection map $(X,Y) \rightarrow Y$ is VC class, which implies the functions $(X,Y) \rightarrow Y\log{F_*(\langle XX', \Sigma \rangle)}$ form a VC class. As $\Sigma_0$ is fixed, then we can easily conclude the class of functions $(X,Y) \rightarrow Y\log{\frac{F_*(X'\Sigma_0 X)}{F_*(X'\Sigma X)}}$ has finite VC dimension. By similar argument we can establish $H^{(2)}_{\delta}$ also has finite VC dimension. Let's say $V_i$ be the VC dimension of $H^{(i)}_{\delta}$. Define $h_{\delta}$ to be envelope function of $H_{\delta}$. Then we have,  
\allowdisplaybreaks
\begin{align*}
|h_{\delta}(X,Y)| & = \left|\sup_{\|\Sigma - \Sigma_0\|_{op} \le \delta} h_{\Sigma}(X,Y)\right| \\
& \le \sup_{\|\Sigma - \Sigma_0\|_{op} \le \delta} |h_{\Sigma}(X,Y)| \\
& \le \sup_{\|\Sigma - \Sigma_0\|_{op} ]\le \delta} \left[\left|\log{F_*(X'\Sigma_0 X)}- \log{F_*(X'\Sigma X)}\right| + \left|\log{(1-F_*(X'\Sigma_0 X))}- \log{(1-F_*(X'\Sigma X))}\right|\right] \\
& \le 2K_1 \sup_{\|\Sigma - \Sigma_0\|_{op} \le \delta} \left|X'(\Sigma - \Sigma_0)X\right|  \le 2K_1U \delta
\end{align*}
Note that, $h_{\delta}$ can also serve as an envelope for both $H^{(1)}_{\delta}$ and $H^{(2)}_{\delta}$. Using the maximal inequality from classical empirical process theory (e.g. see Theorem 2.14.1 in \cite{van1996weak}) we get:  
\begin{equation}
E\left(\sup_{d(\Sigma, \Sigma_0) \le \delta} \left|P_n\left(g_{\Sigma_0} - g_{\Sigma}\right) - P\left(g_{\Sigma_0} - g_{\Sigma}\right)\right|\right) \le J(1, \mathcal{H}_{\delta})\sqrt{Ph^2_{\delta}} \le J(1, \mathcal{H}_{\delta}) 2K_1U \delta\
\end{equation}
for all $\delta > 0$, where 
\begin{align*}
J(1, \mathcal{H}_{\delta}) & = \sup_Q \int_0^1 \sqrt{1 + \log{N(\epsilon\|h_{\delta}\|_{Q,2}, \mathcal{H}_{\delta}, L_2(Q))}} \ d\epsilon \\
& \le \sup_Q \int_0^1 \sqrt{1 + \log{N\left(\epsilon\|h_{\delta}\|_{Q,2}, \mathcal{H}^{(1)}_{\delta}+\mathcal{H}^{(2)}_{\delta}, L_2(Q)\right)}} \ d\epsilon \\
& \le \sup_Q \int_0^1 \sqrt{1 + \sum_{i=1}^2 \log{N\left(\epsilon\|h_{\delta}\|_{Q,2}, \mathcal{H}^{(i)}_{\delta}, L_2(Q)\right)}} \ d\epsilon \\
& \le \sup_Q \int_0^1 \sqrt{1 + \sum_{i=1}^2 \left[\log{K} + \log{V_i} + V_i \log{16e} + 2(V_i - 1)\log{\frac{1}{\epsilon}}\right]} \ d\epsilon
\end{align*}
which is finite. This completes the proof. 
\end{proof}
The last ingradient of the proof is a result due of Massart and Nedelec \cite{massart2006risk}, which, applied to our setting, yields an exponential tail bound. For the convenience of the reader, we present below a tailor-made version of their result which we apply to our problem: 

\begin{theorem}[Application of Talagarand's inequality]
\label{massart}
Let $\{Z_i = (X_i, Y_i)\}_{i=1}^n$ be i.i.d. observations taking values in the sample space $\mathcal{Z} : \mathcal{X} \times \mathcal{Y}$ and let $\mathcal{F}$ be a class of real-valued functions defined on $\mathcal{X}$. Let $\gamma$  be a bounded loss function on $\mathcal{F} \times \mathcal{Z}$ and suppose that $f^* \in \mathcal{F}$ uniquely minimizes the expected loss function $P(\gamma(f, .))$ over  $\mathcal{F}$. Define the empirical risk as $\gamma_n(f) = (1/n) \sum_{i=1}^n \gamma(f, Z_i)$, and $\bar{\gamma}_n(f) = \gamma_n(f) - P(\gamma(f, .))$. Let $l(f^*, f) = P(\gamma(f, .)) - P(\gamma(f^*, .))$ be the excess risk. Assume that: 
\begin{enumerate}
\item We have a pseudo-distance $d$ on $\mathcal{F} \times \mathcal{F}$ satisfying $Var_P[\gamma(f, .) - \gamma(f^*, .)] \le d^2(f,f^*)$. 
\item There exists $F \subseteq \mathcal{F}$ and a countable subset $F' \subseteq F$, such that for each $f \in F$, there is a sequence $\{f_k\}$ of elements of $F'$ satisfying $\gamma(f_k, z) \rightarrow \gamma(f, z)$ as $k \rightarrow \infty$, for every $z \in \mathcal{Z}$. 
\item $l(f,f^*) \ge d^2(f^*, f) \ \forall \ f \in \mathcal{F}$
\item  $\sqrt{n}E\left[\sup_{f \in F': d(f,f^*) \le \sigma}\left[\bar{\gamma}_n(f)- \bar{\gamma}_n(f^*)\right]\right] \le \phi(\sigma)$ for every $\sigma > 0$ such that $\phi(\sigma) \le \sqrt{n} \sigma$. 
\end{enumerate}
Let $\epsilon_*$ be such that $\sqrt{n}\epsilon_*^2 \ge \phi(\epsilon_*)$. Let $\hat{f}$ be the (empirical) minimizer of $\gamma_n$ over $F$ and $l(f^*, F) = \inf_{f \in F}l(f^*, f)$.Then, there exists an absolute constant $K$ such that for all $y \ge 1$, the following inequality holds: $$P\left(l(f^*, \hat{f}) > 2l(f^*, F) + Ky\epsilon_*^2\right) \le e^{-y}$$ 
\end{theorem}
The collection of function is $\mathcal{G} = \{g_{\Sigma}: \|\Sigma - \Sigma_0\|_{op}\}$. The corresponding pseudo-distance is $d(g_{\Sigma}, g_{\Sigma_0}) = \|\Sigma - \Sigma_0\|_{op}$. Condition 2 is easily satisfied as our parameter space has countable dense set and our loss function is continuous with respect to the parameter. Condition 1 and 3 follows form Lemma \ref{variance-bound} and Lemma \ref{curvature_logistic} respectively. Condition 4 is satisfied via Lemma \ref{modulus} with $\phi(\sigma) = \sigma$. Hence, in our case, we can take $\epsilon_n = \sqrt{n}$ and conclude that, there exists a constant K such that, for all $t \ge 1$, $$P\left(n(l(\Sigma_0) - l(\hSigma)) \ge Kt\right) \le e^{-t}$$ From Lemma \ref{curvature_logistic} we have $\|\hSigma - \Sigma_0\|^2_{op} \lesssim l(\Sigma_0) - l(\hSigma)$ which implies $$P\left(\sqrt{n}\|\hSigma - \Sigma_0\|^2_{op} \ge K_1 t\right) \le e^{-t^2}$$ which completes the proof of the theorem. 
\end{proof}

We can combine Theorem \ref{concentration} with Proposition 3.1 and Proposition 3.2 of \cite{yurochkin2020Training} to show that EXPLORE in conjunction with SENSR trains individually fair ML models. For simplicity, we keep the notations same as in \cite{yurochkin2020Training}. Define $\mathcal{L} = \{\ell(\cdot, \theta) \ : \ \theta \in \Theta \}$ as the loss class. We assume that:
\begin{enumerate}
    \item We assume the embeded feature space of $\varphi$ is bounded $R\triangleq\max\{\diam(\varphi),\diam_*(\varphi)\} < \infty$, where $\diam_*$ is the diameter of $\varphi$ in the (unknown) exact fair metric 
\[
d_x^*(x_1,x_2) = \langle (\varphi_1 - \varphi_2) , \Sigma_0 (\varphi_1 - \varphi_2) \rangle^{1/2},
\]
and $\diam$ is the diameter in the learned fair metric
\[
\hat d_x(x_1,x_2) = \langle (\varphi_1 - \varphi_2) , \hSigma (\varphi_1 - \varphi_2) \rangle^{1/2}
.
\]
    \item The loss functions in $\mathcal{L}$ is uniformly bounded, i.e. $0 \le \ell(z, \theta) \le M$ for all $z \in \mathcal{Z}$ and $\theta \in \Theta$ where $z = (x, y)$. 
    \item The loss functions in $\mathcal{L}$ is $L$-Lipschitz with respect to $d_x$, i.e.: 
    \begin{align*}
    & \textstyle\sup_{\theta \in \Theta}\left\{\sup_{(x_1, y), (x_2, y)\in \mathcal{Z}}\left|\ell((x_1, y), \theta) - \ell((x_2, y), \theta)\right|\right\} \le Ld_x(x_1, x_2);
\end{align*}
\end{enumerate}
Define $\delta^*$ to be bias term:
$$\min_{\theta \in \Theta} \sup_{P:W_*(P, P_*) \le \epsilon}\left[\E_{P}\left(\ell(Z,\theta)\right)\right] = \delta^*$$
where $W_*$ is the Wasserstein distance with respect to the true matrix $\Sigma_0$ and $W$ is Wasserstein distance with respect to $\hat \Sigma$. Now for $x_1, x_2 \in \cX$ we have: 
\begin{align*}
       \left|\hat d_x^2(x_1, x_2) - (d^*_x(x_1, x_2))^2\right|& = \left|\left(\varphi_1- \varphi_2\right)^{\top}\left(\hSigma - \Sigma^*\right)\left(\varphi_1- \varphi_2\right)\right| \\
    & \le \|\hSigma - \Sigma^*\|_{op}\|\varphi_1- \varphi_2\|^2_2 \\
    & \le R^2\|\hSigma - \Sigma^*\|_{op} \\
    & \le R^2 K_1 \frac{t}{\sqrt{n}}
\end{align*}
where the last inequality is valid with probability greater than or equal to $1 - e^{-bt^2}$ from Theorem \ref{concentration}. Hence we have with high probability: 
$$\sup_{x_1, x_2 \in \mathcal{X}}\left|\hat d_x^2(x_1, x_2) - (d^*_x(x_1, x_2))^2\right| \le R^2 K_1 \frac{t}{\sqrt{n}}$$ Hence we can take $\delta_c = K_1 t/\sqrt{n}$ in Proposition 3.2 of \cite{yurochkin2020Training} to conclude that: 

\begin{corollary}
\label{cor:EXPLORE+SENSR}
If we assume he loss function $\ell \in \mathcal{L}$ and define the estimator $\hat \theta$ as: 
\[\textstyle
\htheta\in\argmin_{\theta\in\Theta}\sup_{P:W(P,P_n) \le \eps}\Ex_P\big[\ell(Z,h)\big] \,,
\]

then the estimator $\hat \theta$ satisfies with probability greater than or equal to $1 - t-e^{-t^2}$:
\begin{equation}\textstyle
\sup_{P:W_*(P,P_*) \le \eps}\Ex_P\big[\ell(Z,\htheta)\big] - \Ex_{P_*}\big[\ell(Z,\htheta)\big] \le \delta^* + 2\delta_n,
\end{equation}
where $W$ and $W_*$ are the learned and exact fair Wasserstein distances induced by the learned and exact fair metrics (see Section 2.1 in \citet{yurochkin2020Training}) and
\[\textstyle
\delta_n \le \frac{48\mathfrak{C}(\cL)}{\sqrt{n}} + \frac{48LR^2}{\sqrt{n\eps}} + \frac{LK_1 tR^2}{\sqrt{n\eps}} + M\left(\frac{\log\frac2t}{2n}\right)^{\frac12}.
\]
where $\mathfrak{C}(\cL) = \int_0^{\infty} \sqrt{\log{\left(\mathcal{N}_{\infty}\left(\cL, r\right)\right)}} \ dr$,  with $\mathcal{N}_{\infty}\left(\cL, r\right)$ being the covering number of the loss class $\mathcal{L}$ with respect to the uniform metric. 
\end{corollary}

\section{Proofs of Theorems of Section \ref{sec:theory}}
\label{sec:proofsfactorAnalysisTheory}
\subsection{Proof of Theorem \ref{thm:svd_theorem_new}}
\begin{proof}
One key ingredient for the proof is a version of Davis-Kahane's $\sin \Theta$ theorem \citep{davis1970Rotation}, which we state here for convenience: 
\begin{theorem}
\label{thm: DK}
Suppose $A, E \in \mathbb{R}^{d \times d}$. Define $\hat A = A + E$. Suppose $U$ (respectively $\hat U$) denote the top-k eigenvectors of $A$ (respectively $\hat A$). Define $\gamma = \lambda_k(A) - \lambda_{(k+1)}(A)$. Then if $\|E\|_{op} < \gamma$, we have: $$\|\hat U \hat U^T - UU^T\|_{op} \le \frac{\|E\|_{op}}{\gamma - \|E\|_{op}}$$

\end{theorem}

In our context, let's define $U_k$ and $\hat U_k$ denote the eigenspace corresponding to top - $k$ eigenvectors of $\Sigma$ and $S_n$ respectively. Let $\lambda_1 \ge \lambda_2 \ge \dots \ge \lambda_d$ be the eigenvalues of $\Sigma$ and $\hat \lambda_1 \ge \hat \lambda_2 \ge \dots \ge \hat \lambda_d$ be eigenvalues of $S_n$. Applying the above theorem we obtain the following bound:

% which state here for the sake of completeness: 
% \begin{theorem}
% \label{Davis-Kahan}
% Suppose $A \in \mathbb{S}_n$ with $A = V\Lambda V^T$. Suppose $E \in \mathbf{S}_n$ is a symmetric perturbation matrix. Define $\hat A = A + E$ and $\hat A = \hat V \hat \Lambda \hat V^T$. Denote by $\lambda_1 \ge \lambda_2 \ge \dots \ge \lambda_n$ the eigenvalues of $A$ and by $\hat \lambda_1 \ge \hat \lambda_2 \ge \dots \ge \hat \lambda_n$ the eigenvalues of $\hat A$ respectively. Suppose $U$ and $\hat U$ denotes the first $k$ eigenvectors of $A$ and $\hat A$ respectively. If $\delta = \lambda_k - \lambda_{k+1} > 0$ and $\|E\|_{op} < \delta$, then $$\|UU^T - \hat U \hat U^T\|_{op} \le \frac{\|E\|_{op}}{\delta - \|E\|_{op}}$$ 
% \end{theorem}

\begin{equation}
\label{applying_dk}
    \|U_kU_k^* - \hat U_k \hat U_k^*\|_{op} \le \frac{\|\Sigma - S_n\|_{op}}{\eta-\|\Sigma - S_n\|_{op}}
\end{equation}
where $\eta = \lambda_k(
\Sigma) - \lambda_{k+1}(\Sigma)$. To provide a high probability bound on $\|S_n - \Sigma\|_{op}$ we resort to Remark 5.40 (\cite{vershynin2010introduction}), which implies that with probability $\ge 1-2e^{-ct^2}$:

% which we state here for the sake of completeness: 

% \begin{theorem}
% \label{Vershynin}
% Suppose $x_i$'s are i.i.d. sub-gaussian random variable in $\mathbb{R}^d$ with covariance matrix $\Sigma$. Then for every $t \ge 0$, the following inequality holds with probability at-least $1-2e^{-ct^2}$: 
% \begin{equation*}
%     \left\|S_n - \Sigma\right\|_{op} =  \left\|\frac1n X^TX - \Sigma\right\|_{op} \le \delta \vee \delta^2
% \end{equation*}
% where $\delta = \frac{C\sqrt{d} + t}{\sqrt{n}}$, and the constant $C,c$ depends only on $\|x_i\|_{\psi_2}$, the Orlicz-2 norm of $x_i$'s. 
% \end{theorem}
\begin{equation}
    \label{variance_term}
    \|\Sigma - S_n\|_{op} \le \delta \vee \delta^2
\end{equation}
where $\delta = \frac{C\sqrt{d} + t}{\sqrt{n}}$. For $t < (\sqrt{n}\tilde{\gamma} - C\sqrt{d}) \wedge (\sqrt{n\tilde{\gamma}} - C\sqrt{d})$, $\eta > \delta \vee \delta^2$. Hence combining the bounds from equation \ref{applying_dk} and equation \ref{variance_term} we have: 
\begin{equation}
    \label{variance_dk}
    \|U_kU_k^* - \hat U_k \hat U_k^*\|_{op} \le \frac{\delta \vee \delta^2}{\eta-(\delta \vee \delta^2)}
\end{equation}
Here the constant $C,c$ depends only on $\|x_i\|_{\psi_2}$. To conclude the proof, we need a bound on the bias term $\|U_kU_k^T - \tilde{A}_*\tilde{A}_*^T\|_{op}$, which is obtained from another application of Theorem \ref{thm: DK}. From the representation of $\Sigma$ we have: $$\Sigma = A_*A_*^T + B_*B_*^T + \sigma^2 I_d = \tilde{A}_*\Lambda\tilde{A}_*^T + B_*B_*^T + \sigma^2I_d$$ where $\tilde{A}_*$ is the set of eigenvectors of $A_*$ and $\Lambda$ is the diagonal matrix of the eigenvalues. We can apply Theorem \ref{thm: DK} on $\Sigma$ taking $A = \tilde{A}_*\Lambda\tilde{A}_*^T$, $E =  B_*B_*^T + \sigma^2I_d$ and $\Sigma = \hat A$. Here $\lambda_k(A) = \lambda_{min}(A_*A_*^T)$ and $\lambda_{k+1}(A) = 0$. Hence $\gamma = \lambda_{min}(A_*A_*^T)$. As by our assumption $\|B_*B_*^T + \sigma^2I_d\|_{op} < \gamma = \lambda_{min}(A_*A_*^T)$, we obtain   : 
\begin{equation}
\label{bias_term}
    \|U_kU_k^T - \tilde{A}_*\tilde{A}_*^T\|_{op} \le \frac{\|B_*B_*^T + \sigma^2I_d\|_{op}}{\lambda_{min}(A_*A_*^T) - \|B_*B_*^T + \sigma^2I_d\|_{op}} = b
\end{equation}
To conclude the theorem, we provide a bound on $\eta = \lambda_k(\Sigma) - \lambda_{k+1}(\Sigma)$. To upper bound $ \lambda_{k+1}(\Sigma)$ we use Courant-Fisher theorem:
% \begin{equation}
%     \lambda_2(\Sigma) = \inf_{S\subseteq \mathbb{R}^d: \text{dim}(S) = d-1} \sup_{x \in S^{d-1} \cap S} x^T \Sigma x
% \end{equation}
% Using this we get a following upper bound on $\lambda_2(\Sigma)$:
\allowdisplaybreaks
\begin{align*}
    \lambda_{k+1}(\Sigma) = \inf_{S\subseteq \mathbb{R}^d: \text{dim}(S) = d-k} \sup_{x \in S^{d-1} \cap S} x^T \Sigma x & \le \sup_{x \in S^{d-1} \cap \tilde{A}_*^{\perp}} x^T\Sigma x \\
    & = \sup_{x \in S^{d-1} \cap \tilde{A}_*^{\perp}} x^TB_*B_*^Tx + \sigma^2 \le \|B_*B_*^T\|_{op} + \sigma^2
\end{align*}
The lower bound on $\lambda_k(\Sigma)$ can be obtained easily as follows: For any $x \in S^{d-1}$: $$x^T \Sigma x = x^T A_*A_*^T x + x^T B_*B_*^T x + \sigma^2 \ge \lambda_{min}(A_*A_*^T)+\sigma^2$$
This automatically implies $
\lambda_k(\Sigma) \ge \lambda_{min}(A_*A_*^T)+\sigma^2$. Hence combining the bound on $\lambda_k(\Sigma)$ and $\lambda_{k+1}(\Sigma)$ we get:
\begin{equation}
    \label{gamma_bound}\eta = \lambda_k(\Sigma) - \lambda_{k+1}(\Sigma) \ge \lambda_{min}(A_*A_*^T) - \|B_*B_*^T\|_{op} = \tilde{\gamma}
\end{equation} 
 Combining  equation \ref{variance_term}, \ref{bias_term} and \ref{gamma_bound} and using the fact that: 
 $$\left\|\hat U\hat U^{\top} - \tilde A_*\tilde A_*^{\top}\right\|_{op}  = \left\|\hat\Sigma - \Sigma_0\right\|_{op}$$
 we conclude the theorem. 
\end{proof}

\subsection{Proof of Theorem \ref{thm:groupSVD}}
\begin{proof}
The variance covariance matrix of $\varphi_i$ can be represented as following: 
$$\Sigma_{\varphi} = A_*A_*^T + B_*B_*^T + \sigma^2I_d$$ As in the proof of the previous theorem, define $\lambda_1 \ge \dots \ge \lambda_d$  as the eigenvalues of $\Sigma_{\varphi}$ and $\hat \lambda_1 \ge \dots \ge \hat \lambda_d$ as the eigenvalues of $S_n$. Also define by $U_k$ (respectively $\hat U_k$) to be the matrix containing top-k eigenvectors of $\Sigma$ (respectively $S_n$) and $\eta = \lambda_k - \lambda_{k+1}$. Using Davis-Kahan's $\sin \Theta$ theorem (see Theorem \ref{thm: DK}), we conclude that: 
\begin{equation}
    \label{eq:DK_2}
    \|\hat U_k \hat U_k^T - U_kU_K^T\|_{op} \le \frac{\|S_n - \Sigma_{\varphi}\|_{op}}{\gamma - \|S_n - \Sigma_{\varphi}\|_{op}}
\end{equation}
provided that $\eta > \|S_n - \Sigma_{\varphi}\|_{op}$. Using matrix concentration inequality (see remark 5.40 of (\cite{vershynin2010introduction})) we get that with probability $> 1-2e^{-ct^2}$: 
\begin{equation}
    \label{variance_bound_2}
    \|S_n - \Sigma\|_{op} \le \delta \vee \delta^2 + \frac{t}{n}
\end{equation}
where $\delta = (C\sqrt{d} + t)/\sqrt{n}$, for all $t \ge 0$. The difference between this and equation \ref{variance_term} in Theorem \ref{thm:svd_theorem_new} is the extra term $t/n$, which appears due to mean centering the samples. The constants $c, C$ only depends on the $\psi_2$ norm of $\varphi_i$. Combining equation \ref{eq:DK_2} and \ref{variance_bound_2} we conclude that, with high probability we have 
\begin{equation*}
    \|\hat U_k \hat U_k^T - U_kU_k^T\|_{op}  \le \frac{\delta \vee \delta^2 + t/n}{\eta - (\delta \vee \delta^2) - t/n}
\end{equation*}
when $t/n + \delta \vee \delta^2 < \eta$. As before, we apply Theorem \ref{thm: DK} to control the bias. Towards that end, define $A = A_*A_*^T = \tilde{A}_* \Lambda \tilde{A}_*^T$, where $\tilde{A}_*$ is the matrix of eigenvectors of $A_*$ and $\Lambda$ is diagonal matrix with the eigenvalues of $A_*A_*^T$. Also define $E = B_*B_*^T + \sigma^2I_d$ and $\hat A = \Sigma_{\varphi}$. Now, as before, $\lambda_k(A) = \lambda_{min}(A_*A_*^T)$ and $\lambda_{k+1}(A) = 0$. Hence $\gamma = \lambda_k(A) - \lambda_{k+1}(A) = \lambda_{min}(A_*A_*^T)$. Applying Theorem \ref{thm: DK} we conclude: 
\begin{equation} 
    \label{bias_term_2}
    \|U_kU_k^T - \tilde{A}_*\tilde{A}_*^T\|_{op} \le \frac{\|B_*B_*^T + \sigma^2 I_d\|_{op}}{\lambda_{min}(A_*A_*^T) - \|B_*B_*^T + \sigma^2I_d\|_{op}}
\end{equation}
Finally, we use Courant-Fischer Min-max theorem to provide an upper bound on $\eta = \lambda_k(\Sigma_{\varphi}) - \lambda_{k+1}(\Sigma_{\varphi})$. As in the previous proof we have: 
\allowdisplaybreaks
\begin{align*}
    \lambda_{k+1}(\Sigma_{\varphi})
    = \inf_{S\subseteq \mathbb{R}^d: \text{dim}(S) = k+1} \sup_{x \in S^{d-1} \cap S} x^T \Sigma_{\varphi} x & \le \sup_{x \in S^{d-1} \cap \tilde{A}_*^{\perp}} x^T\Sigma_{\varphi} x \\
    & = \sup_{x \in S^{d-1} \cap \tilde{A}_*^{\perp}} x^TB_*B_*^Tx + \sigma^2 \le \|B_* B_*^T\|_{op} + \sigma^2
\end{align*}

\begin{align*}
    \lambda_{k+1}(\Sigma_{\varphi})
    = \sup_{S\subseteq \mathbb{R}^d: \text{dim}(S) = d-k} \sup_{x \in S^{d-1} \cap S} x^T \Sigma_{\varphi} x & \le \sup_{x \in S^{d-1} \cap \tilde{A}_*^{\perp}} x^T\Sigma_{\varphi} x \\
    & = \sup_{x \in S^{d-1} \cap \tilde{A}_*^{\perp}} x^TB_*B_*^Tx + \sigma^2 \le \|B_* B_*^T\|_{op} + \sigma^2
\end{align*}
To get a lower bound on $\lambda_k(\Sigma_{\varphi})$, we use the the other version of Courant-Fischer Minmax theorem: $$\lambda_k(\Sigma_{\varphi}) = \max_{S: dim(S) = d-k+1} \min_{x \in S^{d-1} \cap S} x^T\Sigma x$$
Using this we conclude: $$\lambda_k(\Sigma_{\varphi}) \ge \lambda_{min}(A_*A_*^T) + \sigma^2$$
Hence combining the bound on $\lambda_k(\Sigma_{\varphi})$ and $\lambda_{k+1}(\Sigma_{\varphi})$ we get:
\begin{equation}
    \label{gamma_bound_2}\eta = \lambda_k(\Sigma_{\varphi}) - \lambda_{k+1}(\Sigma_{\varphi}) \ge \lambda_{min}(A_*A_*^T) - \|B_* B_*^T\|_{op} = \tilde{\gamma}
\end{equation} 
Combining equation \ref{variance_bound_2}, \ref{bias_term_2} and \ref{gamma_bound_2} 
and using the fact that: 
 $$\left\|\hat U\hat U^{\top} - \tilde A_*\tilde A_*^{\top}\right\|_{op}  = \left\|\hat\Sigma - \Sigma_0\right\|_{op}$$
 we conclude the theorem. 
\end{proof}

\subsection{Proof of Theorem \ref{thm:provably-fair-training}}
\begin{proof}
The proof of Theorem \ref{thm:provably-fair-training} essentially follows form Proposition 3.2 and Proposition 3.1 of \cite{yurochkin2020Training}. Note that from Theorem \ref{thm:svd_theorem_new}, for any $x_1, x_2 \in \mathcal{X}$: 
\begin{align*}
    \left|\hat d_x^2(x_1, x_2) - (d^*_x(x_1, x_2))^2\right|& = \left|\left(\varphi_1- \varphi_2\right)^{\top}\left(\hSigma - \Sigma^*\right)\left(\varphi_1- \varphi_2\right)\right| \\
    & \le \|\hSigma - \Sigma^*\|_{op}\|\varphi_1- \varphi_2\|^2_2 \\
    & \le R^2\|\hSigma - \Sigma^*\|_{op} \\
    & \le R^2 \left[b + \frac{\delta \vee \delta^2}{\tilde{\gamma} - (\delta \vee \delta^2)}\right]
\end{align*}
where the last inequality is true with probability greater than or equal to $1-2e^{-ct^2}$ from Theorem \ref{thm:svd_theorem_new}. This justifies taking $\delta_c \ge \left[b + \frac{\delta \vee \delta^2}{\tilde{\gamma} - (\delta \vee \delta^2)}\right]$ which along with Proposition 3.1 and 3.2 of \cite{yurochkin2020Training} completes the proof. 
\end{proof}

\addtolength{\tabcolsep}{-3pt}
\begin{table*}[H]
\centering
\caption{Association tests code names}
\vspace{.05in}
\begin{tabular}{lc}
\toprule
FLvINS & Flowers vs. insects \citep{greenwald1998measuring}\\
INSTvWP & Instruments vs. weapons \citep{greenwald1998measuring}\\
MNTvPHS & Mental vs. physical disease \citep{monteith2011implicit}\\
EAvAA & Europ-Amer vs Afr-Amer names \citep{caliskan2017Semantics}\\
EAvAA\citep{bertrand2004Are} & Europ-Amer vs Afr-Amer names \citep{bertrand2004Are}\\
MNvFN & Male vs. female names \citep{nosek2002harvesting}\\
MTHvART & Math vs. arts \citep{nosek2002harvesting}\\
SCvART\citep{nosek2002math} & Science vs. arts \citep{nosek2002math}\\
YNGvOLD & Young vs. old people's names \citep{nosek2002harvesting}\\
\midrule
PLvUPL &  Pleasant vs. unpleasant \citep{greenwald1998measuring} \\
TMPvPRM & Temporary vs. permanent \citep{monteith2011implicit} \\
PLvUPL\citep{nosek2002harvesting} &  Pleasant vs. unpleasant \citep{nosek2002harvesting} \\
CARvFAM & Career vs. family \citep{nosek2002harvesting}\\
MTvFT & Male vs. female terms \citep{nosek2002harvesting}\\
MTvFT\citep{nosek2002math} & Male vs. female terms \citep{nosek2002math}\\
\bottomrule
\end{tabular}
\label{table:association_names}
\end{table*}

\end{document}